\documentclass[preprint,12pt]{elsarticle}

\usepackage{amssymb}
\usepackage{amsmath}
\usepackage{amsthm}
\usepackage{hyperref}
\usepackage[latin1]{inputenc}  
\usepackage[T1]{fontenc}
\usepackage{hyperref}

\usepackage{algorithmicx}
\usepackage[ruled]{algorithm}
\usepackage{algpseudocode}
\usepackage{float}
\usepackage{graphicx}
\usepackage{subcaption}
\usepackage{lineno}

\usepackage{caption}
\newenvironment{myalgorithm}[1]{%
  \noindent\hrule height 0.8pt
  \vspace{2pt}
  \captionsetup{font=small, labelfont=bf, justification=justified, singlelinecheck=false}%
  \captionof{algorithm}{#1}%
  \vspace{2pt}
  \noindent\hrule height 0.8pt
  \vspace{6pt}
}{%
  \vspace{6pt}
  \noindent\hrule height 0.8pt
}

\newtheorem{thm}{Theorem}
\newtheorem{remark}{Remark}

\journal{Elsevier}
\begin{document}

\begin{frontmatter}



\title{Improving Matrix Exponential for Generative AI Flows: A Taylor-Based Approach Beyond Paterson--Stockmeyer}


\author[iteam]{J.~Sastre}
\author[cor2]{D. Faronbi}
\author[i3m]{J.M. Alonso}
\author[cor2]{P. Traver}
\author[imm]{J. Ib\'a\~nez}
\author[idf]{N.~Lloret}
\address[iteam]{Instituto de Telecomunicaciones y Aplicaciones Multimedia}
\address[i3m]{Instituto de Instrumentaci\'{o}n para Imagen Molecular}
\address[imm]{Instituto de Matem\'{a}tica Multidisciplinar}
\address[idf]{Instituto de Dise\~no para la Fabricaci\'on y Producci\'on Automatizada}
\address[cor2]{Music and Audio Research Lab, New York University, New York, USA}
\address{Universitat Polit\`{e}cnica de Val\`{e}ncia, Spain}
\cortext[cor1]{Corresponding author: jsastrem@upv.es. This work has
been supported by the Generalitat Valenciana Grant CIAICO/2023/275.}

\begin{abstract}
\noindent
The matrix exponential is a fundamental operator in scientific computing and system simulation, with applications ranging from control theory and quantum mechanics to modern generative machine learning. While Pad\'e approximants combined with scaling and squaring have long served as the standard, recent Taylor-based methods, which utilize polynomial evaluation schemes that surpass the classical Paterson--Stockmeyer technique, offer superior accuracy and reduced computational complexity. This paper presents an optimized Taylor-based algorithm for the matrix exponential, specifically designed for the high-throughput requirements of generative AI flows. We provide a rigorous error analysis and develop a dynamic selection strategy for the Taylor order and scaling factor to minimize computational effort under a prescribed error tolerance. Extensive numerical experiments demonstrate that our approach provides significant acceleration and maintains high numerical stability compared to existing state-of-the-art implementations. These results establish the proposed method as a highly efficient tool for large-scale generative modeling.
\end{abstract}

\end{frontmatter}


\section{Introduction} \label{introduction}

The computation of matrix functions, particularly the matrix exponential, has been extensively studied due to its wide range of applications in science and engineering~\cite{High08,MoVa03}. Among the various numerical methods developed, Pad\'e approximants have traditionally been favored over Taylor series due to their superior efficiency for comparable accuracy~\cite{High05}. However, despite its widespread adoption, the standard Pad\'e-based approach often incurs unnecessary computational overhead. Recent work, such as \cite{QuickExp20}, has demonstrated that truncated Taylor series can represent a simpler and more efficient alternative, frequently yielding substantial speedups over the Pad\'e-based scaling and squaring implementations found in standard software routines like MATLAB's \texttt{expm}.

In fact, recent advances have shown that Taylor-based methods can
outperform Pad\'e approximants in both accuracy and efficiency. The
Taylor algorithm introduced in~\cite{SIDR15} improved upon classical
Pad\'e-based methods~\cite{High05}, while subsequent
research~\cite{SIDR11b, SIDR14, RSID16} refined the associated error
analysis and reduced computational costs.

A significant breakthrough in this area was the development of
advanced matrix polynomial evaluation techniques~\cite{Sastre18}
that surpass the classical Paterson--Stockmeyer
method~\cite{PaSt73}. These techniques, predicated on structured
multiplications and linear combinations of matrix polynomials,
enable higher-order Taylor approximations with reduced computational
effort. For instance, the method presented in~\cite{SID19}
introduced $m+$ approximations, where the first $m$ terms match the
Taylor series exactly, while additional terms are optimized to
enhance accuracy. These schemes, such as the $m=15+$ case introduced
in \cite[Ex.~5.1]{Sastre18}, achieve greater precision than standard
Taylor polynomials of equivalent computational cost. Furthermore,
recent optimizations for determining the coefficients of evaluation
formulas~\cite{Jar2021} have shown promising results for broader
matrix function computations~\cite{Jar2021, JSI2025}.

The cost of these polynomial approximation methods is dominated by matrix multiplications, denoted by $M$. In practice, these operations scale as $\mathcal{O}(n^3)$ for $n \times n$ matrices, whereas matrix additions and scalar multiplications are of order $\mathcal{O}(n^2)$. In rational methods, matrix inverses are computed by solving linear systems, the cost of which can also be expressed in terms of matrix products as~\cite[App.~C]{BlDo99}:
\begin{equation}\label{equivDM}
D \approx \frac{4}{3} M.
\end{equation}
Consequently, throughout this paper, the cost of both rational and polynomial approximations is quantified in terms of $M$.

Table~\ref{TabCostComp} compares the cost and achievable approximation orders of various state-of-the-art methods, including Pad\'e approximants~\cite{High05}, Paterson--Stockmeyer evaluation~\cite{PaSt73}, advanced matrix polynomial evaluation~\cite{Sastre18, SID19, BBS19} and mixed rational-polynomial schemes~\cite{Sastre12}. The results highlight the superior efficiency of modern polynomial evaluation strategies, which facilitate higher-order approximations using fewer matrix multiplications.

In this paper, we propose an optimized implementation of the matrix exponential based on the matrix polynomial evaluation methods described in~\cite{Sastre18}. Our approach is specifically tailored for the high-throughput requirements of Generative AI flows, where the repeated evaluation of the matrix exponential may constitute a computational bottleneck. Unlike traditional software routines optimized for fixed IEEE precision using pre-computed thresholds \cite{AlHi09,BBS19,SID19,BKS25}, or flexible implementations that depend on external specialized libraries \cite{FaHi19}, we introduce a self-contained framework that dynamically selects the Taylor order and scaling factor for any user-defined tolerance $\varepsilon$. While previous dynamic methods have utilized classical Paterson--Stockmeyer schemes with more complex scaling strategies \cite{CaZi18}, the proposed algorithm leverages higher-order polynomial evaluations with a simplified operational logic. This provides a portable, library-independent solution that balances numerical stability with significant reductions in execution time for large-scale generative models.

\begin{table}[t!]
\centering
{\footnotesize
\begin{tabular}{l|ccccc}
\hline
Polynomial evaluation cost & $3M$ & $4M$ & $5M$ & $6M$ & $7M$ \\
\hline
Approx. order $m$ Paterson--Stockmeyer \cite{PaSt73} & 6 & 9 & 12 & 16 & 20 \\
Approx. order $m$ \cite{BBS19} & 8 & 12 & 18 & 22 & --- \\
Approx. order $m$ \cite{SID19} & 8 & $15+$ & $21+$ & 24 & 30 \\
\hline
Pad\'e evaluation cost & $3.33M$ & $4.33M$ & $5.33M$ & $6.33M$ & $7.33M$ \\
Approx. order Pad\'e method \cite[Tab. 2.2]{High05} & 6 & 10 & 14 & 18 & 26 \\
\hline
Mixed rational polynomial approx. cost & $3.33M$ & $4.33M$ & $5.33M$ & $6M$ & $7M$ \\
Approx. order from method \cite[Tab. 3]{Sastre12} & 9 & 12 & 16 & 21 & 28 \\
\hline
\end{tabular}}
\caption{Cost and available order of approximation with the polynomial evaluation methods from Paterson--Stockmeyer \cite{PaSt73}, Bader--Blanes--Casas \cite[Sec. 4]{BBS19}, Sastre--Ib\'a\~nez--Defez \cite[Sec. 3]{SID19}, the Pad\'e approximation method from Higham \cite[Tab. 2.2]{High05} and the mixed rational and polynomial approximation method from Sastre \cite[Tab. 3]{Sastre12}.}\label{TabCostComp}
\end{table}

Throughout this manuscript, $\mathbb{C}^{n \times n}$ denotes the set of
complex matrices of size $n \times n$, $I$ denotes the identity
matrix for this set, $\rho(A)$ is the spectral radius of matrix $A$,
and $\mathbb{N}$ denotes the set of positive integers. The matrix
norm $\left\|\cdot\right\|$ denotes any subordinate matrix norm; in
particular $\left\|\cdot\right\|_1$ and  $\left\|\cdot\right\|_\infty$ are the 1-norm and the $\infty$-norm, respectively. 
 Due to norm equivalence in finite-dimensional spaces, any submultiplicative matrix norm is valid for scaling criteria and error bounds, with both $\left\|\cdot\right\|_1$ and $\left\|\cdot\right\|_\infty$ requiring only $\mathcal{O}(n^2)$ operations. For consistency with existing implementations in different computing environments, all theoretical formulations and MATLAB codes in this work use the $1$-norm, whereas PyTorch implementations adopt the $\infty$-norm. 
The symbols
$\lceil \cdot \rceil$ and $\lfloor\cdot \rfloor$ denote the smallest
following and the largest previous integer, respectively.

This paper begins in Section~\ref{sec:matrixexpgenerativeflows} by reviewing the role and challenges associated with the matrix exponential in generative flow models. Section~\ref{sec:proposedmethod} introduces the proposed computational methodology, which includes new evaluation formulas for the Taylor matrix polynomial approximation (Subsection~\ref{sec:EFEVPOL}) based on \cite{Sastre18, SaIb21}, a discussion of the scaling and squaring error analysis (Subsection~\ref{sec:erroranalysis}), and the complete algorithmic procedure (Subsection~\ref{sec:algorithm}). Subsequently, Section~\ref{sec:numericalresults} provide the numerical results that assess the algorithm's performance, and Section~\ref{sec:conclusions} presents the conclusions.

\section{The matrix exponential for generative flows}\label{sec:matrixexpgenerativeflows}

\subsection{Generative models}
In recent years, generative models have become very popular. They aim to use training data to learn the probability distribution of a dataset. The resulting model of the distribution can be sampled to obtain new data points according to their estimated probability. This has proven useful in  a variety of tasks, including music generation \cite{copet2023simple, wu2024music}, denoising \cite{bengio2013generalized}, and representation learning \cite{jahanian2021generative}.

There are a variety of approaches to train these probabilistic models. Generative adversarial networks (GANs) \cite{goodfellow2014generative} train two models with adversarial objectives. Variational Autoencoders (VAEs) \cite{kingma2013auto} take a bayesian approach to modelling the data with a prior. Denoising diffusion models \cite{ho2020denoising} aim to model a denoising process (sometimes called score matching) to progressively transform a noise signal into a data point. While these techniques have all been applied succesfully to generative tasks, they unfortunately lack invertibility. Flow generative models were developed to have an invertible mapping from a noise distribution to a data distribution.

\subsection{Generative flows}\label{sec:generativeflows}
For random variable $x \in \mathbb{R}^n$ representing data with probability distribution $p(x)$ and random variable $z \in \mathbb{R}^n$ with a known probability density function $p(z)$, generative flows aim to model an invertible transformation of $x$ into  $z = f(x)$. By the change of variable formula, the relationship between $p(x)$ and $p(z)$ is:

\begin{equation}
    \log p(x) = \log p(z) + \log \left| \frac{\delta f(x)}{\delta x} \right|,
\end{equation}where $\frac{\delta f(x)}{\delta x}$ is the Jacobian of $f$. Assume that there are $k$ number of invertible functions. Then, the log-likelihood of $x$ is given by

\begin{equation}
    \log p(x) = \log p(z) + \sum^{k}_{i=1} \log \left|\frac{\delta h_i}{\delta h_{i-1} } \right|,
\end{equation}where $h_i = f_i(h_{i -1})$.

In generative flow models \cite{xiao2020generative, hoogeboom2020convolution}, the transformation function \( f \) is typically implemented as a composition of invertible functions, often parameterized as \( f = W_m \,\phi(W_{m-1} \,\phi(\cdots \phi(W_1 x))) \), where \( W_i \) are weight matrices and \( \phi \) is a strictly monotonic activation function. Ensuring invertibility of each \( W_i \) during training is challenging, and calculating the determinant of a general matrix is computationally intensive, with a complexity of $O(n^3)$ for an $n \times n$ matrix. A common workaround is to constrain \( W_i \) to be triangular, which simplifies determinant computation but limits expressiveness and parallelizability.

Over the years, there have been multiple approaches to modeling normalizing flows that address this computationally expensive log determinant problem. One solution is to model $f$ autoregresively \cite{papamakarios2017masked}, calculating the determinant in $\mathcal{O}(n)$ time. Contractive residual layers have been used to implement invertible transforms \cite{chen2019residual}. They result in unbiased density estimators at the cost of long computation time. Multiple approaches use coupling layers \cite{kingma2018glow, dinh2016density} to implement unbiased reversible transformations with a triangular matrix for a traceable determinant calculation.

One approach \cite{xiao2020generative, hoogeboom2020convolution} proposed replacing
each weight matrix \( W_i \) with its matrix exponential \( e^{W_i}
\). This guarantees invertibility, improves expressiveness over
triangular matrices, and enables efficient computation of both the
inverse and the log-determinant. Specifically, the inverse is given
by \( (e^{W})^{-1} = e^{-W} \), and the log-determinant simplifies
to \( \log |(e^W)| = \operatorname{Tr}(W) \), reducing the
computational cost from \( \mathcal{O}(n^3) \) to \( \mathcal{O}(n)
\).

The matrix exponential maps \( \mathbb{M}_n(\mathbb{R}) \) to \(
\mathrm{GL}_n(\mathbb{R})^+ \), the group of invertible matrices
with positive determinant. Although not surjective, its image is
rich enough for practical purposes, and any matrix in \(
\mathrm{GL}_n(\mathbb{R})^+ \) can be expressed as a product of
matrix exponentials. This makes the matrix exponential a powerful
tool for constructing expressive and invertible neural architectures in generative modeling. 
Beyond this specific setting, the operator appears across several generative
modeling and sequence learning formulations. First, normalizing flows use
$W = \exp(M)$ to enforce invertibility and compute trace-form log-determinants
\cite{xiao2020generative, hoogeboom2020convolution, lezcano2019cheap},
which requires a matrix exponential, or its action on a vector, per layer and
step, evaluated either densely \cite{xiao2020generative} or as a truncated
series of matrix--vector products \cite{hoogeboom2020convolution}; with
skew-symmetric $M$, the same parametrization yields orthogonal weights in
neural networks \cite{lezcano2019cheap}. Second, continuous-time and diffusion models with
non-diagonal drift adopt the matrix exponential as the flow map or transition
kernel \cite{bilos2021neural, dockhorn2022score, zhang2023fast, zhang2023gddim},
since $e^{tA}$ is the exact flow of the linear dynamics
$\mathrm{d}\mathbf{x}/\mathrm{d}t = A\mathbf{x}$ \cite{MoVa03}. Third,
state-space sequence models rely on the discretization $\bar{A} = \exp(\Delta A)$
\cite{gu2023mamba, smith2023simplified}, but often restrict $A$ to diagonal
form to avoid dense evaluations, and diagonalizing non-normal operators can be
ill-conditioned \cite{yu2024robustifying}. Additionally, Lie group flow
matching uses the matrix exponential as its flow map \cite{sherry2025flow},
and Riemannian flow-map models report its evaluation as a primary
computational cost \cite{davis2026generalised}. Efficient and stable matrix
exponential algorithms therefore benefit architectures that evaluate this
operator directly, and may make dense formulations practical in models
currently designed to avoid it.

\begin{algorithm}
\caption{Algorithm for Computing Matrix Exponential}\label{Alg_expAIorig}
\begin{algorithmic}[1]
\Require Weight matrix $W$, error tolerance $\varepsilon$
\Ensure Matrix exponential $e^W$
\State Choose the smallest non-negative integer $s$ such that $\|W\|_1 / (2^s) < \frac{1}{2}$
\State $W = W / (2^s)$
\State $X = I$
\State $Y = W$
\State $k = 2$
\While{$\|Y\|_1 > \varepsilon$}
    \State $X = X + Y$
    \State $Y = W \cdot Y / k$
    \State $k = k + 1$
\EndWhile
\For{$i = 1$ \textbf{to} $s$}
    \State $X = X \cdot X$
\EndFor
\State \Return $X$
\end{algorithmic}
\end{algorithm}

For the sake of completeness, the algorithm for the computation of the matrix exponential described in \cite[Sec. 3.2]{xiao2020generative} is reproduced here as Algorithm~\ref{Alg_expAIorig}. It relies on the matrix exponential Taylor series expansion \cite[Sec. 3]{MoVa03}, denoted by $T_m(W)$ and given by
\begin{equation}\label{eq:TmW}
T_m(W) = \sum_{i=0}^{m} \frac{W^i}{i!},
\end{equation}where terms are successively summed until the norm of the last term falls below a prescribed tolerance $\varepsilon$. This approach uses a forward absolute error criterion based on an error bound for the Taylor series remainder for approximating the matrix exponential, denoted by $R_m(W)$ and given by
\begin{equation}\label{eq:RmW}
R_m(W)=e^W-T_m(W)=\sum_{i=m+1}^{\infty} \frac{W^i}{i!},
\end{equation}which is bounded by \cite{li66}:
\begin{equation}\label{eq:RmWboundli66}
\|R_m(W)\|_1\leq
\frac{\|W\|_1^{m+1}}{(m+1)!}
\left( \frac{1}{1 - \frac{\|W\|_1}{m+2}} \right).
\end{equation}

Note that the truncated Taylor series is particularly well-suited for scenarios with limited computational resources, restricted access to external libraries or large-scale problems involving well-behaved matrices \cite[p. 52]{QuickExp20}.
According to \cite[Sec. 3.2]{xiao2020generative}, the weight matrices in neural networks, denoted by $W$, are often small in norm such that $\|W\|_1$ is also small. This property makes the incorporation of the matrix exponential into neural network architectures computationally feasible. Exploiting the identity $e^W = (e^{W/2^s})^{2^s}$, in Algorithm \ref{Alg_expAIorig} the matrix is first scaled to ensure that $|W/2^s| < 1/2$, after which the exponential is computed and subsequently squared $2^s$ times. This scaling and squaring strategy further reduces the computational cost while maintaining numerical stability; see \cite[Alg. 1]{xiao2020generative}. The computational cost of evaluating the Taylor approximation of degree $m$ from \eqref{eq:TmW} using \cite[Alg. 1]{xiao2020generative} is measured in terms of matrix multiplications $(M)$. Let $C_{\text{orig}}$ represent the cost of evaluating the unscaled Taylor polynomial in Algorithm \cite[Alg. 1]{xiao2020generative}, which is given by \begin{equation}\label{C_orig}
C_{\text{orig}} = mM.
\end{equation}When scaling is applied, the total cost becomes $(s + m)M$, where $s$ is the scaling parameter. As reported in \cite[Sec. 6.4]{xiao2020generative}, a tolerance of $\varepsilon = 10^{-8}$ is used. Under this setting, the total number of matrix products, given by $s + m$, does not exceed 11, with an average value of 9.28 observed during training \cite[Tab. 6]{xiao2020generative}.

Cost \eqref{C_orig} makes Algorithm \ref{Alg_expAIorig} unsuitable for high-dimensional data \cite[Sec. 3.2]{xiao2020generative}. To address this limitation, a second method is proposed that combines matrix exponentials with neural networks, using a low-rank parameterization of the weight matrix $W = A_1 A_2$, where $A_1 \in \mathbb{R}^{n \times t}$ and $A_2 \in \mathbb{R}^{t \times n}$.

By defining $V = A_2 A_1$, and using the associative property of matrix multiplication, the matrix exponential $e^W$ can be approximated via a truncated series expansion of $V$:

\begin{equation}\label{eq:expmWV}
e^W \approx I + A_1 \left( \sum_{i=0}^{m} \frac{V^i}{(i+1)!} \right) A_2,
\end{equation}where the truncation error bound of the matrix series of \( V \), denoted here by $R'_m(V)$, is given by (8) from \cite[Sec. 3.2]{xiao2020generative} reproduced here

\begin{equation}\label{eq:RmV}
\left\| R'_m(V) \right\|_1=\left\| \sum_{i = m+1}^{\infty} \frac{V^i}{(i + 1)!} \right\|_1 \leq
\left( \frac{\|V\|_1^{m+1}}{(m + 2)!} \right)
\left( \frac{1}{1 - \frac{\|V\|_1}{m + 3}} \right).
\end{equation}

This reduces the computational cost to $\mathcal{O}(t^3)$ since $V \in \mathbb{R}^{t \times t}$. The rank of $W$ is at most $t$, so choosing a smaller $t$ improves efficiency but may reduce expressiveness. Thus, there is a trade-off between computational cost and model capacity.

The approximation method is similar to Algorithm \ref{Alg_expAIorig}, with modifications: set the scale coefficient $s := 0$, change line 4 to $Y := W/2$, and line 5 to $k := 3$ \cite[Sec. 3.2]{xiao2020generative}.

While the cost of the algorithms proposed in \cite[Sec. 3.2]{xiao2020generative} represent a reasonable computational cost, further reductions are possible by optimizing the polynomial evaluation strategy.

\section{Proposed method}\label{sec:proposedmethod}

Our implementation is specifically tailored for this generative flow setting. It builds upon efficient polynomial evaluation techniques, extending beyond the classical Paterson--Stockmeyer method as introduced in \cite{Sastre18}, to substantially reduce the number of matrix products required for a specified approximation order. Furthermore, the algorithm incorporates a dynamic procedure that optimally selects the Taylor series degree ($m$) and the scaling factor ($s$) to minimize the overall computational cost. This selection ensures that the error tolerance is maintained, chosen to be no less than the unit roundoff ($u$) of the arithmetic used, while also strategically reducing the risk of overscaling \cite[Sec. 1]{AlHi09}.

Numerical experiments confirm that our approach consistently achieves lower average computational costs than the baseline method in \cite{xiao2020generative} and standard library implementations such as \texttt{torch.linalg.matrix\_exp}, particularly in the context of generative AI flows, where evaluations of thousands of matrix exponentials are required. This improvement translates into faster training times and reduced resource consumption, making our method suitable for large-scale machine learning applications.

\subsection{Evaluation formulas of the Taylor-based matrix exponential approximations}\label{sec:EFEVPOL}

This work adopts the evaluation formulas from \cite{SID19} to compute Taylor-based approximations for orders up to $m=15+$. Table~\ref{TabCostComp} demonstrates that these schemes significantly outperform the Paterson--Stockmeyer method; for a fixed computational budget, they achieve orders $8$ and $15+$, whereas the classical approach only reaches orders $6$ and $9$, respectively. Moreover, these schemes represent the most efficient approximations identified in Table~\ref{TabCostComp}, outperforming both the state-of-the-art Taylor \cite{BBS19} and Pad\'e \cite{High05} methods. Although \cite{SID19} also provides formulas for higher-order Taylor approximations, given the low error tolerance required in our application, these higher-order variants do not offer a reduction in computational cost and are therefore excluded from consideration. For completeness, the formulas corresponding to the selected approximation orders are reproduced below:
{\setlength\arraycolsep{2pt}{
\begin{eqnarray}\label{T_1_4}
T_1(A)&=&A+I,\\
T_2(A)&=&A^2/2+A+I,\\
T_4(A)&=&((A^2/4+A)/3+I)A^2/2+A+I,
\end{eqnarray}}}
{\setlength\arraycolsep{2pt}{
\begin{eqnarray}\label{m8s2by0}
y_{02}(A)&=&A^2(c_1A^2+c_2A),\\\label{m8s2by1}
T_8(A)&=&(y_{02}(A)+c_3A^2+c_4A)(y_{02}(A)+c_5A^2)\\
& & +c_6y_{02}(A)+A^2/2+A+I,\nonumber
\end{eqnarray}}}where the coefficients $c_i$, for $i = 1,\,2\,,\ldots,\,6$, are listed in Table~\ref{Tab_m8s2}, rounded to IEEE double-precision floating-point arithmetic.

\begin{table}[t!]
\begin{center}
{\footnotesize
\begin{tabular}{l|cccc}\hline
$c_1$&4.980119205559973$\times 10^{-3}$\\
$c_2$&1.992047682223989$\times 10^{-2}$\\
$c_3$&7.665265321119147$\times 10^{-2}$\\
$c_4$&8.765009801785554$\times 10^{-1}$\\
$c_5$&1.225521150112075$\times 10^{-1}$\\
$c_6$&2.974307204847627$\times 10^{0\phantom{-}}$\\
\hline
\end{tabular}}
\end{center}
\caption{Coefficients for the Taylor approximation of the matrix exponential of order $m = 8$, computed using formulas (\ref{m8s2by0}) and (\ref{m8s2by1})}\label{Tab_m8s2}
\end{table}
{\setlength\arraycolsep{2pt}{
\begin{eqnarray}\label{m16s2by0}
y_{02}(A)&=&A^2(c_{1}A^2+c_{2}A),\\ \label{m16s2by1}
y_{12}(A)&=&(y_{02}(A)+c_{3}A^2+c_{4}A)(y_{02}(A)+c_{5}A^2)+c_{6}y_{02}(A)+c_{7}A^2,\\
y_{22}(A)&=&(y_{12}(A)+c_8A^2+c_9A)(y_{12}(A)+c_{10}y_{02}(A)+c_{11}A)\nonumber\\\label{m16s2by2}
& &+c_{12}y_{12}(A)+c_{13}y_{02}(A)+c_{14}A^2+c_{15}A+c_{16}I,
\end{eqnarray}}}where the coefficients $c_i$, for $i = 1,\,2,\,\ldots,\,16$, are listed in Table~\ref{Tab_m16s2}, also rounded to IEEE double-precision floating-point arithmetic.

\begin{table}[t!]
\begin{center}
{\footnotesize
\begin{tabular}{l|cccc}\hline
$c_{1}$&\phantom{-}4.018761610201036$\times 10^{-4}$&\vline&          $c_{9}$&\phantom{-}2.224209172496374$\times 10^{0\phantom{-}}$\\
$c_{2}$&\phantom{-}2.945531440279683$\times 10^{-3}$&\vline&          $c_{10}$&-5.792361707073261$\times 10^{0\phantom{-}}$\\
$c_{3}$&          -8.709066576837676$\times 10^{-3}$&\vline&          $c_{11}$&-4.130276365929783$\times 10^{-2}$\\
$c_{4}$&\phantom{-}4.017568440673568$\times 10^{-1}$&\vline&          $c_{12}$&\phantom{-}1.040801735231354$\times 10^{1\phantom{-}}$\\
$c_{5}$&\phantom{-}3.230762888122312$\times 10^{-2}$&\vline&          $c_{13}$& -6.331712455883370$\times 10^{1\phantom{-}}$\\
$c_{6}$&\phantom{-}5.768988513026145$\times 10^{0\phantom{-}}$&\vline&$c_{14}$&\phantom{-}3.484665863364574$\times10^{-1}$\\
$c_{7}$&\phantom{-}2.338576034271299$\times 10^{-2}$          &\vline&$c_{15}$&1\\
$c_{8}$&\phantom{-}2.381070373870987$\times 10^{-1}$          &\vline&$c_{16}$&1\\
\hline
\end{tabular}}
\end{center}
\caption{Coefficients of $y_{02}$, $y_{12}$, and $y_{22}$ from equations (\ref{m16s2by0})-(\ref{m16s2by2}) used in the Taylor-based approximation of the matrix exponential of order $m = 15+$.}\label{Tab_m16s2}
\end{table}
Note that
\begin{enumerate}
   \item $T_1(A)$ and $T_2(A)$ in \eqref{T_1_4} correspond to the direct Taylor evaluation formulas of orders 1 and 2, respectively, while $T_4(A)$ corresponds to the Paterson--Stockmeyer evaluation scheme \cite{PaSt73} for order $m = 4$.
    \item The cost of evaluating $T_8(A)$ using formulas \eqref{m8s2by0} and \eqref{m8s2by1} is $3M$: one multiplication for computing $A^2$, one for evaluating \eqref{m8s2by0}, and one for evaluating \eqref{m8s2by1}. In contrast, computing the Taylor approximation of degree $m=8$ using the original algorithm from \cite[Alg. 1]{xiao2020generative}, as given by \eqref{C_orig}, requires $8M$, which is more than twice as expensive.
    \item Recall that, in exact arithmetic, equations (15) and (16) from \cite{SID19} show that $y_{22}(A)$ in \eqref{m16s2by2} satisfies
\begin{equation}\label{y22b16}
y_{22}(A) = T_{15}(A) + b_{16}A^{16},
\end{equation}
and, therefore, the approximation remainder is given by
\begin{equation}\label{eq:RmW15+}
R_{15+}(W) = e^W - y_{22}(W) = \left(\frac{1}{16!} - b_{16}\right)W^{16} + \sum_{i=17}^{\infty} \frac{W^i}{i!},
\end{equation}
where, from \eqref{m16s2by0}--\eqref{m16s2by2}, it can be shown that the coefficient is $b_{16} = c_1^4$. Using the value of $c_1$ from Table \ref{Tab_m16s2}, we obtain
\begin{equation}\label{b16value}
b_{16}=c_1^4 \approx 2.608368698098256 \times 10^{-14}.
\end{equation}
This value differs from the exact Taylor coefficient $1/16!$ by a relative error of approximately $|b_{16} - 1/16!| \cdot 16! \approx 0.454$. Consequently, $y_{22}(A)$ provides a more accurate approximation than $T_{15}(A)$ in exact arithmetic, although it remains less accurate than the full Taylor approximation $T_{16}(A)$ as demonstrated in \cite{SID19}. Approximations of this type are referred to in \cite{SID19} as $T_{15+}(A)$.
    Similarly, the evaluation of $y_{22}(A)$ via formulas \eqref{m16s2by0}--\eqref{m16s2by2} requires only $4M$ matrix products. This cost contrast is even more pronounced when using the low-rank parameterization from \eqref{eq:expmWV}. Since no scaling is applied in that case, the evaluation of the corresponding matrix polynomial of degree 15 requires $15M$ products. Consequently, the baseline approach is 3.75 times more expensive than our proposed method in this configuration.
\end{enumerate}

\subsection{Error analysis} \label{sec:erroranalysis}

This section details the self-contained framework for the dynamic selection of the Taylor order $m$ and the scaling parameter $s$. Unlike fixed-precision algorithms that rely on pre-computed thresholds for specific IEEE formats or error tolerances \cite{AlHi09,BBS19,SID19,BKS25}, the proposed approach derives these parameters directly for a user-specified tolerance $\varepsilon$. By integrating advanced polynomial evaluation formulas into the error bound analysis, we achieve a more streamlined selection logic than previous dynamic methods \cite{CaZi18}.

In particular, we provide a refined version of the absolute error analysis from \eqref{eq:RmWboundli66} specifically for nonnormal matrices. For such matrices, the relation (1.5) from \cite{AlHi09} states that
\begin{equation}
\rho(A) \leq \|A^k\|^{1/k} \leq \|A\|, \quad k = 1, 2, \dots, \infty.
\end{equation}
Consequently, the inequality
\begin{equation}
\|A^k\| \leq \|A\|^{k}, \quad k = 1, 2, \dots, \infty,
\end{equation}
can be significantly strict, i.e., $\|A^k\| \ll \|A\|^{k}$, as discussed in \cite[Sec.~1]{AlHi09}. This property is exploited in \cite[Th.~1.1]{SIDR14} to derive tighter bounds for the matrix power series, ensuring efficiency without the overhead of overscaling. For completeness of the exposition, the theorem is reproduced below:

\begin{thm} \label{theorem1}
Let
$h_{l}(x)=\sum_{k \geq l} b_{k} x^{k}$
be a power series with radius of convergence $R$, and let
$\tilde{h}_{l}(x)=\sum_{k \geq l} |b_{k}| x^{k}.$
For any matrix $A \in \mathbb{C}^{n \times n}$ with $\rho(A)<R$,
if $a_{k}$ is an upper bound for $||A^{k}||$ $(||A^{k}||\leq
a_{k})$, $p \in \mathbb{N}$, $1\leq p \leq l$, $p_0\in \mathbb{N}$
is the multiple of $p$ with $l\leq p_0\leq l+p-1$, and
\begin{equation}\label{alphap}
\alpha_{p}=\max\{a_{k}^{\frac{1}{k}}:k=p,l,l+1,l+2,\ldots,p_0-1,p_0+1,p_0+2,\ldots,l+p-1\},
\end{equation}then
\begin{equation}\label{hlbound}
||h_{l}(A)||\leq \tilde{h}_{l}(\alpha_{p}).
\end{equation}
\end{thm}

Using Theorem \ref{theorem1} the following theorem provides a sharper bound than \eqref{eq:RmWboundli66} given originally by \cite{li66} for the remainder $R_m(W)$ of the matrix exponential Taylor series of nonnormal matrices.

\begin{thm}\label{theorem2}
Let $A \in \mathbb{C}^{n \times n}$ be a square matrix and let $m \in \mathbb{N}$. Suppose that for each $k \in \mathbb{N}$, there exists an upper bound $a_k$ such that $\|A^k\|_1 \leq a_k$. Let $p \in \mathbb{N}$ with $1 \leq p \leq m+1$, and let $p_0$ be the multiple of $p$ satisfying $m+1 \leq p_0 \leq m + p$. Define
\begin{equation}\label{eq:alphap}
\alpha_p = \max\left\{ a_k^{1/k} : k = p, m+1, m+2, \ldots, p_0 - 1, p_0 + 1, p_0 + 2, \ldots, m + p \right\}.
\end{equation}
Then the remainder of the matrix exponential Taylor approximation of degree $m$,
\begin{equation}\label{eq:Rmexpm}
R_m(A) = e^A - \sum_{k=0}^{m} \frac{A^k}{k!} = \sum_{k=m+1}^{\infty} \frac{A^k}{k!},
\end{equation}
satisfies the bound
\begin{equation}\label{eq:RmAbound}
\|R_m(A)\|_1 < \frac{\alpha_p^{m+1}}{(m+1)!} \cdot \left( \frac{1}{1 - \frac{\alpha_p}{m+2}} \right),
\end{equation}
provided that condition
\begin{equation}\label{alphacondition}
\alpha_p < m+2,
\end{equation}is fulfilled.
\end{thm}
\begin{proof}
Let $A \in \mathbb{C}^{n \times n}$ be a square matrix and let $m \in \mathbb{N}$. The remainder of the matrix exponential Taylor approximation of degree $m$ is given by \eqref{eq:RmW}
\[
R_m(A) = e^A - T_m(A) = \sum_{k=m+1}^{\infty} \frac{A^k}{k!}.
\]

Define the scalar power series
\[
h_{m+1}(x) = \sum_{k=m+1}^{\infty} \frac{x^k}{k!} = \tilde{h}_{m+1}(x) = \sum_{k=m+1}^{\infty} \left|\frac{1}{k!}\right| x^k,
\]
which has infinite radius of convergence. Suppose that for each $k \in \mathbb{N}$, there exists an upper bound $a_k\geq 0$ such that $\|A^k\|_1 \leq a_k$.

Let $p \in \mathbb{N}$ with $1 \leq p \leq m+1$, and let $p_0$ be the multiple of $p$ satisfying $m+1 \leq p_0 \leq m + p$. Define
\[
\alpha_p = \max\left\{ a_k^{1/k} : k = p, m+1, m+2, \ldots, p_0 - 1, p_0 + 1, p_0 + 2, \ldots, m + p \right\}.
\]

By Theorem~\ref{theorem1}, we have
\begin{equation}\label{eq:Rmhtilde}
\|R_m(A)\|_1 = \left\| \sum_{k=m+1}^{\infty} \frac{A^k}{k!} \right\|_1 \leq \sum_{k=m+1}^{\infty} \frac{\|A^k\|_1}{k!} \leq \sum_{k=m+1}^{\infty} \frac{a_k}{k!} \leq \tilde{h}_{m+1}(\alpha_p).
\end{equation}

To bound $\tilde{h}_{m+1}(\alpha_p)$, note that factoring out the first term one gets
\[
\tilde{h}_{m+1}(x) = \frac{x^{m+1}}{(m+1)!} \cdot \sum_{k=0}^{\infty} \frac{x^k}{(m+2)(m+3)\cdots(m+1+k)}.
\]

Note that for each \( k \geq 0 \),
\[
(m+2)(m+3)\cdots(m+1+k) > (m+2)^k,
\]
so for $x> 0$ it follows that
\[
\frac{x^k}{(m+2)(m+3)\cdots(m+1+k)} < \left( \frac{x}{m+2} \right)^k.
\]

Therefore,
\[
\tilde{h}_{m+1}(x) < \frac{x^{m+1}}{(m+1)!} \cdot \sum_{k=0}^{\infty} \left( \frac{x}{m+2} \right)^k.
\]

This is a geometric series with ratio \( \frac{x}{m+2} \), and if \( \frac{x}{m+2}< 1 \) it converges to
\[
\sum_{k=0}^{\infty} \left( \frac{x}{m+2} \right)^k = \frac{1}{1 - \frac{x}{m+2}}.
\]

Hence,
\begin{equation}\label{eq:liou66}
\sum_{k=m+1}^{\infty} \frac{x^k}{k!} = \tilde{h}_{m+1}(x) < \frac{x^{m+1}}{(m+1)!} \cdot \frac{1}{1 - \frac{x}{m+2}}, \quad \text{for } 0 < x < m+2.
\end{equation}

Applying \eqref{eq:Rmhtilde} and \eqref{eq:liou66} with $x = \alpha_p < m+2$ then \eqref{eq:RmAbound} holds.
\end{proof}

In the following, The matrix exponential scaling parameter, $s$, is determined based on the result of Theorem \ref{theorem2} to ensure that the Taylor series remainder remains below a predefined error tolerance. Let $W$ be the weight matrix for generative flows from Algorithm \ref{Alg_expAIorig}. Let $a_k$ be an upper bound  for $\|W^k\|_1$ such that $\|W^k\|_1 \leq a_k$. Let $p \in \mathbb{N}$ with $1 \leq p \leq m+1$, and let $p_0$ be the multiple of $p$ satisfying $m+1 \leq p_0 \leq m + p$. Let $\alpha_p$ be defined by \eqref{eq:alphap}.
Let's scale the the matrix $W$ by the scaling parameter $s$ so that
\begin{equation}\label{eq:RmWboundTol2}
\frac{\|(W/2^s)^{m+1}\|_1}{(m+1)!}\leq\frac{(\alpha_p/2^s)^{m+1}}{(m+1)!}\leq
\varepsilon,
\end{equation}where $\varepsilon$ is the selected user error tolerance from Algorithm \ref{Alg_expAIorig}. Note that $\varepsilon$ must fulfill
\begin{equation}\label{eq:expsilon_geq_u}
\varepsilon\geq u,
\end{equation}where $u$ is the machine unit roundoff error in the working precision arithmetic, since errors below the machine roundoff error can't be guaranteed in the working precision arithmetic computations. Then, if $s$ fulfills \eqref{eq:RmWboundTol2} one gets
\begin{equation}\label{eq:sR1}
s \geq \frac{1}{m+1} \log_2\left( \frac{\alpha_p^{m+1}}{(m+1)! \cdot \varepsilon} \right),
\end{equation}and then it follows that
\begin{equation}\label{eq:sR2}
s=\left\lceil \frac{1}{m+1} \log_2\left( \frac{\alpha_p^{m+1}}{(m+1)! \cdot \varepsilon} \right) \right\rceil.
\end{equation}Since $s$ is the lowest integer satisfying \eqref{eq:sR2}, then scaling the matrix $W$ by $s$ given by \eqref{eq:sR2} it follows that
\begin{equation}\label{eq:RmWboundTol3}
\frac{\|(W/2^s)^{m+1}\|_1}{(m+1)!}\leq\frac{(\alpha_p/2^s)^{m+1}}{(m+1)!}=\varepsilon'\leq
\varepsilon,
\end{equation}and by \eqref{eq:RmAbound} from Theorem \ref{theorem2}, if $\alpha_p/2^s<m+2$ one gets
\begin{equation}\label{eq:RmWbound2}
\|R_m(W/2^s)\|_1\leq
\varepsilon' \frac{1}{1 - \frac{\alpha_p/2^s}{m+2}},
\end{equation}where $\alpha_p/2^s=((m+1)!\varepsilon')^{\frac{1}{{m+1}}}\leq ((m+1)!\varepsilon)^{\frac{1}{{m+1}}}$ and then
\begin{equation}\label{eq:RmWbound3}
\|R_m(W/2^s)\|_1\leq
\varepsilon \frac{1}{1 - \frac{((m+1)!\varepsilon)^{1/{m+1}}}{m+2}}.
\end{equation}

We verify the fulfillment of condition \eqref{alphacondition} by utilizing the error tolerance $\varepsilon = 10^{-8}$ reported in \cite[Sec. 6.4]{xiao2020generative} and the selected Taylor degrees $m \in \{1, 2, 4, 8, 15\}$ (as defined in Section \ref{sec:EFEVPOL}). For these parameters, it is easy to show that the requirement is met if the maximum value of the bound, $\alpha_p/2^s$, satisfies
$$\alpha_p/2^s \leq (((m+1)!\varepsilon))^{\frac{1}{m+1}} < m+2.$$
Solving for $\varepsilon$, this condition holds provided

\begin{equation}
\varepsilon < \frac{(m+2)^{m+1}}{(m+1)!}.
\end{equation}

Since the right-hand side is strictly increasing with $m$, its minimum value of $4.5$ is attained at $m=1$. Consequently, the condition is satisfied for all $m \ge 1$ whenever $\varepsilon < 4.5$, and thus holds true for $\varepsilon = 10^{-8}$. Subsequently, Analyzing the maximum total error bound given by \eqref{eq:RmWbound3}, we find that for $\varepsilon = 10^{-8}$ it equals
\begin{equation}
\|R_m(W/2^s)\|_1\leq 1.145\varepsilon,
\end{equation}
attained at $m=15$. Since this bound is directly proportional to and on the order of the target tolerance $\varepsilon$, it demonstrates that the total error is tightly controlled by $\varepsilon$, thereby validating the scaling parameter selection derived from \eqref{eq:sR2} under practical floating-point arithmetic. Evaluating bound \eqref{eq:RmWbound3} for $m \in \{1, 2, 4, 8, 15\}$ and machine tolerances $\varepsilon \in \{2^{-8}, 2^{-11}, 2^{-24}, 2^{-53}\}$, corresponding to Brain Floating Point (BF16), IEEE 754 half precision (FP16), single precision (FP32), and double precision (FP64) formats, respectively, we find that for each fixed $\varepsilon$, the maximum bound is consistently attained at $m=15$. Furthermore, these bounds decrease monotonically as $\varepsilon$ grows smaller, yielding
\begin{equation}
\|R_m(W/2^s)\|_1 \le 1.395\varepsilon,
\end{equation}
for BF16 ($\varepsilon = 2^{-8}$), and
\begin{equation}
\|R_m(W/2^s)\|_1 \le 1.042\varepsilon,
\end{equation}
for FP64 ($\varepsilon = 2^{-53}$). 
Hence, the total error remains tightly bounded on the order of $\varepsilon$ across practical floating-point precisions.

To ensure that the bound is accomplished properly, the proposed algorithm will compute the scaling parameter $s$ by using the term $m+1$ as stated here, and also the term $m+2$,  selecting the largest scaling parameter among both of them. For the case $m=15+$, where the term $m+1=16$ is not exactly $1/16!$, selecting the maximum scaling parameter obtained with $m+1$ and also the term $m+2=17$, for which the exponential remainder term corresponds effectively to $1/17!$, will ensure a proper scaling parameter value.

The error analysis for the series in \eqref{eq:expmWV} is analogous to the error analysis described here, and the following theorem can be proved similarly:
\begin{thm}\label{theorem3}
Let $A \in \mathbb{C}^{n \times n}$ be a square matrix and let $m \in \mathbb{N}$. Suppose that for each $k \in \mathbb{N}$, there exists an upper bound $a_k$ such that $\|A^k\|_1 \leq a_k$. Let $p \in \mathbb{N}$ with $1 \leq p \leq m+1$, and let $p_0$ be the multiple of $p$ satisfying $m+1 \leq p_0 \leq m + p$. Define
\begin{equation}\label{eq:alphapV}
\alpha_p = \max\left\{ a_k^{1/k} : k = p, m+1, m+2, \ldots, p_0 - 1, p_0 + 1, p_0 + 2, \ldots, m + p \right\}.
\end{equation}
Then the remainder of the matrix exponential Taylor approximation of degree $m$,
\begin{equation}\label{eq:RmexpmV}
R'_m(A) = \sum_{k=m+1}^{\infty} \frac{A^k}{(k+1)!},
\end{equation}
satisfies the bound
\begin{equation}\label{eq:RmAboundV}
\|R'_m(A)\|_1 < \frac{\alpha_p^{m+1}}{(m+2)!} \cdot \left( \frac{1}{1 - \frac{\alpha_p}{m+3}} \right),
\end{equation}
provided that condition
\begin{equation}\label{alphaconditionV}
\alpha_p < m+3,
\end{equation}is fulfilled.
\end{thm}

\begin{remark}\label{remark1}
The proposed approach offers computational advantages for nonnormal matrices. Consider the well-known test matrix (1.2) from \cite{AlHi09}, defined as
\[
W = \begin{bmatrix} 1 & b \\ 0 & -1 \end{bmatrix},
\]
taking for instance $b = 10$. For this matrix, the baseline Algorithm~\ref{Alg_expAIorig} selects a scaling parameter $s = 5$, whereas it is easy to show that our approach, based on Theorem \ref{theorem2}, yields $s = 0$. This reduction eliminates 5 matrix-matrix multiplications, showing that our method may prevent the over-scaling typically induced by large off-diagonal entries. Furthermore, unlike the approach in \cite{AlHi09}, this is achieved without estimating norms of matrix powers. Such estimations incur an $\mathcal{O}(n^2)$ overhead that remains non-negligible for moderate-sized matrices, as noted in \cite{BBS19} for $n\leq 150$ and \cite[p.~213]{SID19} for $n=128$.
\end{remark}

\begin{remark}\label{remark2expm_normal}
Let $\theta_m$ denote the threshold value of the matrix norm for an approximation of degree $m$ to satisfy a given tolerance $\varepsilon$. For the degree progression $m \in \{1, 2, 4, 8, 15\}$, the inequality $\theta_m > \frac{1}{2} \theta_{m_{\text{next}}}$ determines whether applying a scaling step to halve the norm and using degree $m$ is computationally cheaper than evaluating the approximation directly at the next available degree $m_{\text{next}}$. This inequality holds for coarse tolerances such as $\varepsilon = 10^{-1}$, but fails systematically for tight tolerances ($\varepsilon \le 10^{-2}$). We verified this behavior numerically for $\varepsilon = 10^{-1}, 10^{-2}, \dots, 10^{-16}$. Mathematically, the threshold values are defined by $\theta_m = (\varepsilon / C_m)^{1/(m+1)}$, where $C_m = 1/(m+1)!$ for $m < 15$ and $C_{15} = 1/16! - b_{16}$ from \eqref{b16value}. As $\varepsilon \to 0$, the decay rate of $\theta_m$ is governed by $\varepsilon^{1/(m+1)}$. Because the exponent $1/(m+1)$ is smaller for higher-degree terms, $\theta_m$ decreases significantly faster for lower values of $m$, causing the ratio $\theta_m / \theta_{m_{\text{next}}}$ to drop below $1/2$ as the tolerance tightens. Moreover, for $\varepsilon=10^{-1}$ the maximum ratio $\theta_m / \theta_{m_{\text{next}}} \approx 0.53$, obtained for $m=1$, is only slightly larger than $0.5$.

A similar behavior is observed when evaluating the degree progression $m \in \{2, 3, 5, 9, 16\}$ associated with the remainder bound in Theorem \ref{theorem3}. Even with the shifted factorial coefficients, in this case the condition $\theta_m > \frac{1}{2} \theta_{m_{\text{next}}}$ fails systematically for tolerances $\varepsilon \le 10^{-3}$ due to the same asymptotic dominance of the fractional exponent. Specifically, the maximum ratio $\theta_m / \theta_{m_{\text{next}}} \approx 0.68$ for $\varepsilon = 10^{-1}$ and $\theta_m / \theta_{m_{\text{next}}} \approx 0.56$ for $\varepsilon = 10^{-2}$, obtained for the minimum value of $m=2$. Consequently, for any target accuracy below $\varepsilon = 10^{-3}$, and even for $\varepsilon = 10^{-1}$ and $10^{-2}$, scaling combined with any degree lower than the maximum available is highly unlikely to provide a lower computational cost.
\end{remark}

 Finally, regarding the approximation in \eqref{eq:expmWV}, we set the scaling parameter to $s=0$, consistent with the approach described in \cite[Sec. 3.2]{xiao2020generative}. For Taylor based approximations of approximation orders $m \in \{1, 2, 4, 8, 15+\}$, we employ the efficient evaluation formulas detailed in Section \ref{sec:EFEVPOL}. For degrees exceeding these values, we utilize the classical Paterson--Stockmeyer method.

\subsection{Proposed algorithms}\label{sec:algorithm}
Taking into account the evaluation formulas from Section \ref{sec:EFEVPOL} and the error analysis from Section \ref{sec:erroranalysis}, Algorithm~\ref{Algorithm2} computes the matrix exponential using the scaling and squaring
algorithm and the Taylor polynomial approximation. Polynomials are efficiently evaluated by means of Paterson--Stockmeyer method or evaluation formulas \eqref{T_1_4}-\eqref{m16s2by2}. 

\begin{algorithm}
\caption{Algorithm for Computing Matrix Exponential}\label{Algorithm2}
\begin{algorithmic}[1]
\Require Weight matrix $W$, error tolerance $\varepsilon$
\Ensure Matrix exponential $e^W$
\State Obtain the polynomial order $m$ and the scaling parameter $s$ by using Algorithms 3 or 4
\State $W = W / (2^s)$
\State Evaluate the matrix polynomial $X = P_m(W)$ by Paterson--Stockmeyer technique or by evaluation formulas \eqref{T_1_4}-\eqref{m16s2by2}
\For{$i = 1$ \textbf{to} $s$}
    \State $X = X \cdot X$
\EndFor
\State \Return $X$
\end{algorithmic}
\end{algorithm}

The algorithm is fundamentally suitable for any error tolerance $\varepsilon$ that is greater than or equal to the unit roundoff error. However, note that when employing IEEE double precision arithmetic, where  $u=2^{-53}\approx 1.11 \times  10^{-16}\ll 10^{-8}$ from \cite[Sec. 6.4]{xiao2020generative}, to achieve a higher level of accuracy that approaches the machine precision the algorithm parameters should be adjusted. Specifically, the accuracy is improved by selecting more terms in the Taylor approximation as shown in \cite[Sec. 6]{SID19}.

Initially, Algorithm~\ref{Algorithm2} determines the optimal parameters $m$ and $s$ by invoking Algorithm \ref{Algorithm3} or Algorithm \ref{Algorithm4}. Next, the matrix $W$ is scaled following the scaling and squaring approximation to reduce its norm. Then, the matrix polynomial is evaluated either using the Paterson--Stockmeyer method or evaluation formulas \eqref{T_1_4}-\eqref{m16s2by2}. Finally, the exponential of the matrix $W$ is recovered carrying out the appropriate $s$ matrix products corresponding to the squaring stage.

In the implemented functions derived from Algorithms \ref{Algorithm2}, \ref{Algorithm3}, and \ref{Algorithm4}, log$_2$ has been used to avoid the appearance of large numbers when computing high powers of $\alpha_p$ and factorials in expressions like \eqref{eq:sR2}, especially if low precision arithmetic is used.

As said above, the algorithm for evaluating \eqref{eq:expmWV} is analogous to Algorithm~\ref{Algorithm2}, setting $s=0$ and using similar evaluation formulas to those in Section \ref{sec:EFEVPOL} for Taylor based approximations with $m=1,\,2,\,4,\,8,\,15+$ and the Paterson--Stockmeyer method for Taylor approximations of degrees higher than $15+$.

Algorithm \ref{Algorithm3} determines the polynomial order $m$ and the scaling parameter $s$ required to compute the matrix exponential of $W$ using a Taylor approximation of order $m \leq 16$ and a desired tolerance $\varepsilon$. It is assumed that the polynomials are evaluated using the Paterson--Stockmeyer method.
\vspace{6pt}
\begin{myalgorithm}{Algorithm for Obtaining a Polynomial Order $m$ and a Scaling Parameter $s$. The polynomial will be evaluated via Paterson--Stockmeyer}\label{Algorithm3}
{\small
\begin{algorithmic}[1]
\Require Weight matrix $W$, error tolerance $\varepsilon$
\Ensure Polynomial order $m$ and scaling parameter $s$
\State $m = 0, s = 0$

\If {$\|{W}\|_1 = 0$}
    \State \Return $m, s$
\EndIf
\State $M = [1, 2, 4, 6, 9, 12, 16]$
\State $J = \left\lceil \sqrt{M} \right\rceil$
\State $K = M./J$
\State $C = [1. / [2!, 3!, 3!, 4!, 5!, 6!, 7!, 8!, 10!, 11!, 13!, 14!, 17!, 18!]]$
\State $finish=0$
\State $i=1$
\While {$i \leq \text{length}(M)$ \textbf{and} finish = 0}
    \State $m=M_i$
    \State $j=J_i$
    \State $k=K_i$
    \If {$j > k$}
        \State $W^j=W^{j-1} \cdot W$
    \EndIf
    \State $p = 2\cdot i - 1$
    \If {$m = 1$}
        \State $E_1 = C_p \cdot \|{W}\|_1^2$
        \State $E_2 = C_{p+1} \cdot \|{W}\|_1^3$
    \Else
        \State $E_1 = C_p \cdot \|{W^j}\|_1^k \cdot \|{W}\|_1 $
        \State $E_2 = C_{p+1} \cdot \|{W^j}\|_1^k \cdot \|{W}^2\|_1 $
    \EndIf
    \If {$E_1 + E_2 \leq \varepsilon$}
        \State $finish = 1$
    \Else
        \State $i = i + 1$
    \EndIf
\EndWhile
\If {$finish=0$}
    \For {$i = 1$ \textbf{to} $2$}
        \State $s1=\left\lceil{\log_2(E_i/\varepsilon)/(m+i)} \right\rceil$
        \If {$s1>s$}
            \State $s=s1$
        \EndIf
    \EndFor
\EndIf
\If {$s>30$}
    \State $s=30$
\EndIf
\State \Return $m, s$
\end{algorithmic}
}
\end{myalgorithm}
\vspace{6pt}

Initially, if $W$ is a null matrix, the values of $m$ and $s$ equal to 0 will be returned. Next, vector $M$ stores the possible degrees of the considered polynomials, vector $J$ is completed as the power of $W$ that must be computed for each value of $M$, and vector $K$ is filled in according to the quotient of the elements of $M$ divided by the elements of $J$.

For each possible value $m$ in $M$, vector $C$ stores two consecutive values corresponding to the coefficients of the Taylor series located at positions $m+1$ and $m+2$, that is, $1/(m+1)!$ and $1/(m+2)!$.

Next, from lines 9 to 31, an attempt is made to determine the most appropriate degree $m$ of the polynomial to use, from the point of view of accuracy and computational efficiency. Since $\frac{W^{m+1}}{(m+1)!}$ and $\frac{W^{m+2}}{(m+2)!}$ are the first two terms of the remainder for the Taylor-based approximation of order $m$ to the exponential of $W$, our objective is to obtain the value of $m$ that guarantees that:
\begin{equation}
\frac{\|W^{m+1}\|_1}{(m+1)!}+\frac{\|W^{m+2}\|_1}{(m+2)!}\leq \varepsilon,
\label{eq:error_considered}
\end{equation}
where $E_1$ corresponds to the first summand and $E_2$ to the second one in the algorithm.

Notwithstanding, since $W^{m+1}$ and $W^{m+2}$ are not explicitly calculated, the
1-norm of these matrix powers is bounded using the 1-norm of the highest
power of $W$ computed thus far. For example, when $m=1$ (lines 19--21), just $W$ is
known, but $\|W^2\|_1$ and $\|W^3\|_1$ are required; these are bounded by
$\|W\|_1^2$ and $\|W\|_1^3$, respectively. Similarly, if, for example, $m=12$ (lines 22--24),
$W^4$ will have been previously calculated, and the required norms $\|W^{13}\|_1$ and
$\|W^{14}\|_1$ are bounded by $\|W^4\|_1^3 \cdot \|W\|_1$ and
$\|W^4\|_1^3 \cdot \|W^2\|_1$. Thus, if $m$ equals 12, $j$ and $k$ take the
values 4 and 3, accordingly. In detail, Table \ref{table:mjk_algorithm3} contains the values of $j$ and $k$ used for any value of $m$ taken into account.

If expression (\ref{eq:error_considered}) is not satisfied for any considered value of $m$, then $m$ is set equal to the maximum degree in vector $M$, i.e., $m=16$. Consequently, the minimum value of $s$ is obtained such that

\begin{equation}
\max\left(\frac{1}{(m+1)!} \cdot \left\|\left(\frac{W}{2^s}\right)^{m+1}\right\|_1,  \frac{1}{(m+2)!} \cdot \left\|\left(\frac{W}{2^s}\right)^{m+2}\right\|_1\right) \leq \varepsilon.
\label{eq:obtain_s}
\end{equation}

By solving for $s$ in the first and second terms of (\ref{eq:obtain_s}) independently and taking the logarithm of each, we obtain

\begin{equation}
s\geq \max\left(\frac{1}{m+1}\cdot \log_2\left(\frac{\|W^{m+1}\|_1}{\varepsilon \cdot (m+1)!} \right),
\frac{1}{m+2}\cdot \log_2\left(\frac{\|W^{m+2}\|_1}{\varepsilon \cdot (m+2)!} \right) \right),
\end{equation} in accordance with lines 32 to 39 in the algorithm. Finally, to avoid overscaling, the maximum value of $s$ is set to 30 in lines 40 and 41.

\begin{table}[H]
\begin{center}
\begin{tabular}{|c|| c|| c| c|| c| c||}
\hline
$M$ & Computed $W$ Power & \multicolumn{2}{|c||}{Required Norms} & \multicolumn{2}{|c||}{Estimated Norms} \\\hline\hline
1   &  $W$ & $\|W^2\|_1$ & $\|W^3\|_1$ &  $\|W\|^2_1$ & $\|W\|^3_1$ \\
\hline\hline
2   &  $W^2$ & $\|W^3\|_1$ & $ \|W^4\|_1$ &  $\|W^2\|^1_1 \cdot \|W\|_1$ & $ \|W^2\|^1_1 \cdot \|W^2\|_1$ \\
\hline
4   &  $W^2$ & $\|W^5\|_1$ & $ \|W^6\|_1$ &  $\|W^2\|^2_1 \cdot \|W\|_1$ & $ \|W^2\|^2_1 \cdot \|W^2\|_1$ \\
\hline
6   &  $W^3$ & $\|W^7\|_1$ & $ \|W^8\|_1$ &  $\|W^3\|^2_1 \cdot \|W\|_1$ & $ \|W^3\|^2_1 \cdot \|W^2\|_1$ \\
\hline
9   &  $W^3$ & $\|W^{10}\|_1$ & $ \|W^{11}\|_1$ &  $\|W^3\|^3_1 \cdot \|W\|_1$ & $ \|W^3\|^3_1 \cdot \|W^2\|_1$ \\
\hline
12   &  $W^4$ & $\|W^{13}\|_1$ & $ \|W^{14}\|_1$ &  $\|W^4\|^3_1 \cdot \|W\|_1$ & $ \|W^4\|^3_1 \cdot \|W^2\|_1$ \\
\hline
16   &  $W^4$ & $\|W^{17}\|_1$ & $ \|W^{18}\|_1$ &  $\|W^4\|^4_1 \cdot \|W\|_1$ & $ \|W^4\|^4_1 \cdot \|W^2\|_1$ \\
\hline
$m$   &  $W^j$ & $\|W^{m+1}\|_1$ & $ \|W^{m+2}\|_1$ &  $\|W^j\|^k_1 \cdot \|W\|_1$ & $ \|W^j\|^k_1 \cdot \|W^2\|_1$ \\
\hline
\end{tabular}
\end{center}
\caption{Computed powers of $W$, the corresponding norms required by \eqref{eq:error_considered}, and their resulting approximations for each polynomial degree $M$ considered in Algorithm \ref{Algorithm3}.}\label{table:mjk_algorithm3}
\end{table}

Similarly, Algorithm \ref{Algorithm4} fulfills the same role as Algorithm~\ref{Algorithm3}, but incorporates formulas \eqref{T_1_4}-\eqref{m16s2by2} for evaluating the matrix polynomials. While the structure remains consistent with the previous algorithm, the polynomial degrees differ. Furthermore, the maximum explicitly computed power of $W$ is restricted to $W^2$. A key distinction lies in the penultimate value of vector $C$, which corresponds to the Taylor-based approximation $T_{15+}$ defined in \eqref{m16s2by2}. As previously noted, this approximation is equivalent to \eqref{y22b16}, where $b_{16}$ is given in IEEE double-precision arithmetic by \eqref{b16value}. Consequently, the penultimate element of vector $C$ is derived from the corresponding remainder term \eqref{eq:RmW15+}. Table \ref{table:mjk_algorithm4} incorporates the values of $j$ and $k$ applied for any value of $m$. For $m = 15$, the indices $j$ and $k$ are set to 2 and 8, respectively. Thus, using $W^2$, the bounds for $\|W^{16}\|_1$ and $\|W^{17}\|_1$ (lines 22, 23, and 28) are bounded by $\|W^2\|_1^8$ and $\|W^2\|_1^8 \cdot \|W\|_1$. The calculation of the scaling parameter $s$ follows the procedure described in Algorithm~\ref{Algorithm3}. However, determining $s$ based on the maximum value from the first two terms of the remainder \eqref{eq:RmW15+} ensures in this case a more conservative selection.

\vspace{6pt}
\begin{myalgorithm}{Algorithm for Obtaining a Polynomial Order $m$ and a Scaling Parameter $s$.
The polynomial will be evaluated via formulas \eqref{T_1_4}--\eqref{m16s2by2}}
\label{Algorithm4}
{\small
\begin{algorithmic}[1]
\Require Weight matrix $W$, error tolerance $\varepsilon$
\Ensure Polynomial order $m$ and scaling parameter $s$
\State $m = 0, s = 0$
\If {$\|{W}\|_1 = 0$}
    \State \Return $m, s$
\EndIf
\State $M = [1, 2, 4, 8, 15]$
\State $J = [1, 2, 2, 2, 2]$
\State $K = \left\lceil M./J\right\rceil$
\State $C=[1/2!, 1/3!, 1/3!, 1/4!, 1/5!, 1/6!, 1/9!, 1/10!, |1/16! - b_{16}|, 1/17!]$ \Comment{where $b_{16}$ is given by \eqref{b16value}}
\State $finish=0$
\State $i=1$
\While {$i \leq \text{length}(M)$ \textbf{and} finish = 0}
    \State $m=M_i$
    \State $j=J_i$
    \State $k=K_i$
    \If {$m = 2$}
        $W^2=W^{j-1} \cdot W$
    \EndIf
    \State $p = 2\cdot i - 1$
    \If {$m = 1$}
        \State $E_1 = C_p \cdot \|{W}\|_1^2$
        \State $E_2 = C_{p+1} \cdot \|{W}\|_1^3$
    \ElsIf {$j\cdot  k = m$} \% m=2, 4, 8
        \State $E_1 = C_p \cdot \|{W^j}\|_1^k \cdot \|{W}\|_1$
        \State $E_2 = C_{p+1} \cdot \|{W^j}\|_1^k \cdot \|{W}^2\|_1$
    \Else \ \% m=15
        \State $E_1 = C_p \cdot \|{W^j}\|_1^k $
        \State $E_2 = C_{p+1} \cdot \|{W^j}\|_1^k \cdot \|{W}\|_1$
    \EndIf
    \If {$E_1 + E_2 \leq \varepsilon$}
        \State $finish = 1$
    \Else
        \State $i = i + 1$
    \EndIf
\EndWhile
\If {$finish=0$}
    \For {$i = 1$ \textbf{to} $2$}
        \State $s1=\left\lceil{\log_2(E_i/\varepsilon)/(m+i)} \right\rceil$
        \If {$s1>s$}
            \State $s=s1$
        \EndIf
    \EndFor
\EndIf
\If {$s>30$}
    \State $s=30$
\EndIf
\State \Return $m, s$
\end{algorithmic}
}
\end{myalgorithm}

\begin{table}[H]
\begin{center}
\begin{tabular}{|c|| c|| c| c|| c| c||}
\hline
$M$ & Computed $W$ Power & \multicolumn{2}{|c||}{Required Norms} & \multicolumn{2}{|c||}{Estimated Norms} \\
\hline\hline
1   &  $W$ & $\|W^2\|_1$ & $\|W^3\|_1$ &  $\|W\|^2_1$ &  $\|W\|^3_1$ \\
\hline\hline
2   &  $W^2$ & $\|W^3\|_1$ & $\|W^4\|_1$ &  $\|W^2\|^1_1 \cdot \|W\|_1$ & $\|W^2\|^1_1 \cdot \|W^2\|_1$ \\
\hline
4   &  $W^2$ & $\|W^5\|_1$ & $\|W^6\|_1$ &  $\|W^2\|^2_1 \cdot \|W\|_1$ & $\|W^2\|^2_1 \cdot \|W^2\|_1$ \\
\hline
8   &  $W^2$ & $\|W^{9}\|_1$ & $\|W^{10}\|_1$ &  $\|W^2\|^4_1 \cdot \|W\|_1$ & $\|W^2\|^4_1 \cdot \|W^2\|_1$ \\
\hline
even $m$   &  $W^j$ & $\|W^{m+1}\|_1$ & $\|W^{m+2}\|_1$ &  $\|W^j\|^k_1 \cdot \|W\|_1$ & $\|W^j\|^k_1 \cdot \|W^2\|_1$ \\
\hline\hline
15   &  $W^2$ & $\|W^{16}\|_1$ & $\|W^{17}\|_1$ &  $\|W^2\|^8_1$ & $\|W^2\|^8_1 \cdot \|W\|_1$ \\
\hline
odd $m$   &  $W^j$ & $\|W^{m+1}\|_1$ & $\|W^{m+2}\|_1$ &  $\|W^j\|^k_1$ & $\|W^j\|^k_1 \cdot \|W\|_1$ \\
\hline
\end{tabular}
\end{center}
\caption{Computed powers of $W$, the norms required by \eqref{eq:error_considered}, and their corresponding approximations for each polynomial degree $M$ considered in Algorithm \ref{Algorithm4}. Note that $j=2$ for all degrees $m > 1$.}\label{table:mjk_algorithm4}
\end{table}

\vspace{1em}
\noindent \textit{Cost of the selected pair $(m,s)$.} 
To verify whether the selection made by Algorithm~\ref{Algorithm4} can be improved, we compared the pair $(m,s)$ it returns against the optimal pair that minimizes the total number of matrix multiplications, $c(m)+s$, across all candidate degrees $m \in \{1, 2, 4, 8, 15^+\}$. The respective costs are $c(1)=0M$, $c(2)=1M$, $c(4)=2M$, $c(8)=3M$, and $c(15^+)=4M$, which includes the initial evaluation of $W^2$. For each candidate degree, $s$ was chosen as the smallest scaling parameter ensuring that both remainder estimates in Algorithm~\ref{Algorithm4} (see Table~\ref{table:mjk_algorithm4}) fall below the prescribed tolerance, meaning both selections rely on identical estimates. 

The numerical comparison spanned $\log_2\|W\|$ from $-30$ to $14$ in increments of $0.01$, $\log_2\|W^2\|$ from $2\log_2\|W\|-40$ to $2\log_2\|W\|$ in increments of $0.25$, and tolerances $\varepsilon=10^{-k}$ for $k=1,\ldots,16$, alongside $\varepsilon=2^{-24}$. No matrices $W$ were generated for this comparison: Algorithm~\ref{Algorithm4} depends on $W$ only through $\|W\|$ and $\|W^2\|$, so both the selected and the optimal pair are functions of these two norms and of $\varepsilon$, and were evaluated directly on this grid, which therefore covers every matrix whose norms lie within it. The upper limit $\|W^2\|\le\|W\|^2$ follows from submultiplicativity, and values of $\|W^2\|$ well below $\|W\|^2$ correspond to strongly nonnormal matrices. This generated approximately $1.2 \times 10^7$ test cases. For every $\varepsilon \le 10^{-2}$, including the specific tolerance $\varepsilon=10^{-8}$ used in \cite[Sec. 6.4]{xiao2020generative}, the pair returned by Algorithm~\ref{Algorithm4} achieved the strictly minimal cost in all instances. Only at the exceptionally coarse tolerance of $\varepsilon=10^{-1}$ did $474$ out of $7.0 \times 10^5$ cases admit a pair that was cheaper by a single matrix product, with $430$ of these exceptions requiring scaling. These results align with the behavior described in Remark~\ref{remark2expm_normal}.

\section{Numerical results}\label{sec:numericalresults}

This section compares the numerical and computational properties of
the different implementations. 
The code is available at \url{https://github.com/hipersc/MatrixFunctions/}.

In all experiments, the normwise relative error serves as our primary metric to quantify the accuracy of each matrix exponential method relative to the ``exact'' solution (computed to double-precision limits, as detailed below). This error is defined as:

\begin{equation}\label{eq:normwise_relative_error}
\textmd{Er}(A)=\frac{\|\exp(A)-\widetilde{\exp}(A)\|_2}{\|\exp(A)\|_2},
\end{equation}
where $\exp(A)$ and $\widetilde{\exp}(A)$ stands for the ``exact'' and the approximated solution. 

All numerical comparisons were conducted in MATLAB (version 2025a) and PyTorch. The MATLAB evaluation (Section~\ref{sec:numericalresultsmatlab}) benchmarks the proposed approach against the variable precision scaling and squaring method from \cite{BKS25}, alongside other algorithms. Specifically, we compare the following five MATLAB implementations:

\begin{enumerate}
\item \texttt{expm\_flow\_ps}: Based on the proposed Algorithms~\ref{Algorithm2} and \ref{Algorithm3}, this approach evaluates matrix polynomials using the Paterson--Stockmeyer technique.
\item \texttt{expm\_flow\_opt15}: Identical to the previous implementation, but utilizing the optimized evaluation formulas \eqref{T_1_4}--\eqref{m16s2by2} for matrix polynomials (Algorithm~\ref{Algorithm4}).
\item \texttt{expm\_flow}: The original algorithm from \cite{xiao2020generative}, strictly rewritten in MATLAB to ensure a fair comparison within an identical computational environment.
\item \texttt{expm\_1diagpad}: A variable precision scaling and squaring algorithm relying on diagonal Pad\'e approximants \cite{BKS25}, where the polynomial order and scaling parameter are determined solely by the 1-norm of the target matrix.
\item \texttt{expm\_1}: A variable precision scaling and squaring algorithm utilizing Taylor, diagonal (restricted to $r_{13/13}$), and non-diagonal Pad\'e approximants \cite{BKS25}. Similar to \texttt{expm\_1diagpad}, the polynomial order and scaling parameter depend strictly on the matrix 1-norm.
\end{enumerate}

For the GPU-accelerated PyTorch experiments, we benchmark PyTorch implementations of the proposed algorithms against standard library implementations. Because direct PyTorch implementations of the algorithms from \cite{AlHi09} and \cite{BKS25} are unavailable, we compare our approach against  \path{scipy.linalg.expm}, which implements the algorithm from \cite{AlHi09} (Section~\ref{sec:comparisonlibraries}) and \texttt{torch.linalg.matrix\_exp}, which uses the Taylor polynomial approximation from \cite{BBS19} (Sections~\ref{sec:comparisonlibraries} and~\ref{sec:weightsize}).

To isolate the specific performance contributions of our proposed method, Sections~\ref{sec:weightsize} and \ref{sec:middleground} introduce an intermediate baseline, that we have denoted as \texttt{expm\_flow\_opt15\_fixed}. The computational advantages of the proposed implementations \texttt{expm\_flow\_opt15}  over standard implementations arise from two distinct factors: the novel evaluation formulas \eqref{T_1_4}--\eqref{m16s2by2}, which minimizes the matrix products of the approximations, and the dynamic selection of the parameters $(m,s)$ based on Theorem~\ref{theorem2}. To disentangle these effects, the intermediate baseline employs the evaluation formula for $m=15+$ from \eqref{m16s2by2}, but fixes the scaling rule to
\begin{equation}\label{eq:sfixed}
s=\max\left(0,\left\lceil\log_2\frac{\|W\|}{\theta}\right\rceil\right),\qquad
\theta=\left(\varepsilon\,16!\right)^{1/16}.
\end{equation}
This rule is derived strictly from the first term of the Taylor remainder, bypassing the dynamic selection of $m$. Consequently, the computational cost of \texttt{expm\_flow\_opt15\_fixed} is exactly $(4+s)M$. Comparing our full method against this baseline allows us to precisely quantify the computational savings yielded by the error bounds established in Theorem~\ref{theorem2}.

\subsection{Numerical results in MATLAB}\label{sec:numericalresultsmatlab}

In this subsection, all computational tests were conducted on a PC equipped with an Intel Core i9-14900HX processor, 32 GB of RAM, and an NVIDIA GeForce RTX 4070 graphics card.

The numerical experiments evaluate the MATLAB implementation of the proposed method \texttt{expm\_flow\_opt15} against distinct primary benchmarks: \texttt{expm\_flow}, i.e. the MATLAB implementation of the algorithm presented in \cite{xiao2020generative} (Algorithm~\ref{Alg_expAIorig}), the custom implementation \texttt{expm\_flow\_ps} of the Paterson--Stockmeyer (P--S) evaluation scheme, and the two codes \texttt{expm\_1diagpad} and \texttt{expm\_1} from \cite{BKS25} that utilize Taylor or Pad\'e approximants for variable precision. This selection is predicated on the current state of deep learning frameworks such as PyTorch, where native matrix exponential routines (e.g., \texttt{torch.linalg.matrix\_exp}) are typically optimized for standard IEEE precision formats and do not support user-defined error tolerances. To evaluate the performance of our proposed methods in a realistic setting, we adopt the framework from \cite{xiao2020generative} as a baseline for matrix exponential applications in generative flow models. Besides, the inclusion of the P--S scheme as a secondary benchmark allows for a direct quantification of the efficiency gains attributable specifically to the optimized polynomial evaluation formulas integrated into the proposed algorithm. Furthermore, integrating two external codes enables benchmarking against state-of-the-art variable-precision methods~\cite{BKS25}. While the implementations in \cite{BKS25} are restricted to a fixed set of predefined tolerances ($2^{-11}$, $2^{-24}$, $10^{-4}$, $10^{-8}$, $10^{-12}$, and $10^{-16}$), our proposed approach accommodates arbitrary target tolerances down to machine precision. Here, $2^{-11}$ and $2^{-24}$ correspond to FP16 and FP32 unit roundoffs, respectively. 

To evaluate the performance of these five codes, a testbed composed of 567 ill-conditioned matrices belonging to the Matrix Computation Toolbox (MCT)~\cite{higham2002test} and the Eigtool MATLAB Package (EMP)~\cite{wright2009eigtool} was generated. Their size, always a power of 2, ranged from 2 to 1024.

The ``exact'' solution to the exponential function for each matrix $A$ was calculated using the \texttt{expm\_mp} function provided in \cite{FaHi19}. It can employ diagonal Pad\'e approximants or Taylor approximants with a Schur decomposition or in transformation-free form. The function offers arbitrary precision and depends on Advanpix Multiprecision Computing Toolbox \cite{advanpix}.

Let $E_1$ and $E_2$ be the matrix exponential computed by means of two different methods available in \cite{FaHi19}, where 256 decimal digits were be used in all subsequent multiple-precision calculations. Matrix $E_1$ was considered the ``exact''  exponential of $A$ if the following condition is satisfied:
\[\frac{{{{\left\| {{E_1} - {E_2}} \right\|}_2}}}{{{{\left\| {{E_1}} \right\|}_2}}} \le u, \]
where $u=2^{-53}$ stands for the unit roundoff in IEEE double precision arithmetic (FP64). Otherwise, matrix $A$ was excluded. 

Graphically, Figure \ref{fig_set1} shows the results obtained from the first of the experiments, carried out on the matrices comprising MCT and EMP. A tolerance of $2^{-11}$ (FP16) was used for the results displayed on the left-hand side of the graphs and of $2^{-24}$ (FP32) for the ones appearing on the right-hand side. 

Figures \ref{fig_set1_a} and \ref{fig_set1_b} plot the normwise relative error committed by each of the codes for the different matrices analyzed. The black line appearing in them corresponds to the expected relative errors when calculating the exponential of the matrices. This value is equivalent to the product of the condition number of the exponential function for each matrix multiplied by the error tolerance used ($2^{-11}$ or $2^{-24}$) in our experiments. Ideally, it is to be expected that all results will be below this line, with those ones further down clearly being better. It is easy to see how the results for \texttt{expm\_flow} are less accurate than those for the other codes, as they occupy the highest positions in the graph.

Figures \ref{fig_set1_c} and \ref{fig_set1_d} correspond to the performance profile. It represents, for each value $\alpha$ on the $x$-axis, the percentage of matrices, expressed as a proportion of one, for which the error incurred in calculating their exponential function is $\alpha$ times less than or equal to the minimum error committed by each of the methods compared. The higher the values provided, the better the code. For single precision ($u = 2^{-24}$, FP32), Figure \ref{fig_set1_d} displays comparable behavior across all functions except \texttt{expm\_flow}, which yields the poorest results. Similarly, for half precision ($u = 2^{-11}$, FP16), Figure \ref{fig_set1_c} shows consistent performance, with a slight advantage for \texttt{expm\_flow\_opt15} and \texttt{expm\_1}.

Finally, Figures~\ref{fig_set1_e}--\ref{fig_set1_h} report the required number of matrix products and the corresponding execution times (in seconds) for all test matrices. Across both half precision (FP16) and single precision (FP32), the results exhibit a consistent trend. Specifically, Figures~\ref{fig_set1_e} and \ref{fig_set1_f} show that \texttt{expm\_flow\_opt15} required the fewest matrix products (3662 for FP16 and 4079 for FP32), whereas \texttt{expm\_flow} required the most (6184 and 7673). Relative to the baseline \texttt{expm\_flow\_opt15}, the product-count ratios for FP16 and FP32 are 1.16 and 1.19 for \texttt{expm\_flow\_ps}, 1.48 and 1.45 for \texttt{expm\_1}, 1.52 and 1.48 for \texttt{expm\_1diagpad}, and 1.69 and 1.88 for \texttt{expm\_flow}. Overall, \texttt{expm\_flow\_opt15} is the most efficient method, followed by \texttt{expm\_flow\_ps}, while the remaining algorithms require at least 1.45 to 1.88 times as many matrix multiplications.

Execution times (Figures~\ref{fig_set1_g} and \ref{fig_set1_h}) further confirm the superior efficiency of \texttt{expm\_flow\_opt15}, which required $5.97$~s for FP16 and $6.78$~s for FP32. In comparison, \texttt{expm\_flow\_ps} required $6.14$~s and $7.22$~s ($1.03$ and $1.07$ times slower), while \texttt{expm\_1diagpad} exhibited the largest overhead at $10.45$~s and $11.18$~s ($1.75$ and $1.65$ times slower). The slight discrepancy between matrix product ratios and runtime ratios occurs because the test suite contains smaller matrix dimensions (from $2 \times 2$ up to $1024 \times 1024$), where multiplication counts do not entirely dominate overhead.

\begin{figure}[H]
    \begin{subfigure}[b]{0.48\textwidth}
        \centerline{\includegraphics[scale=0.35]{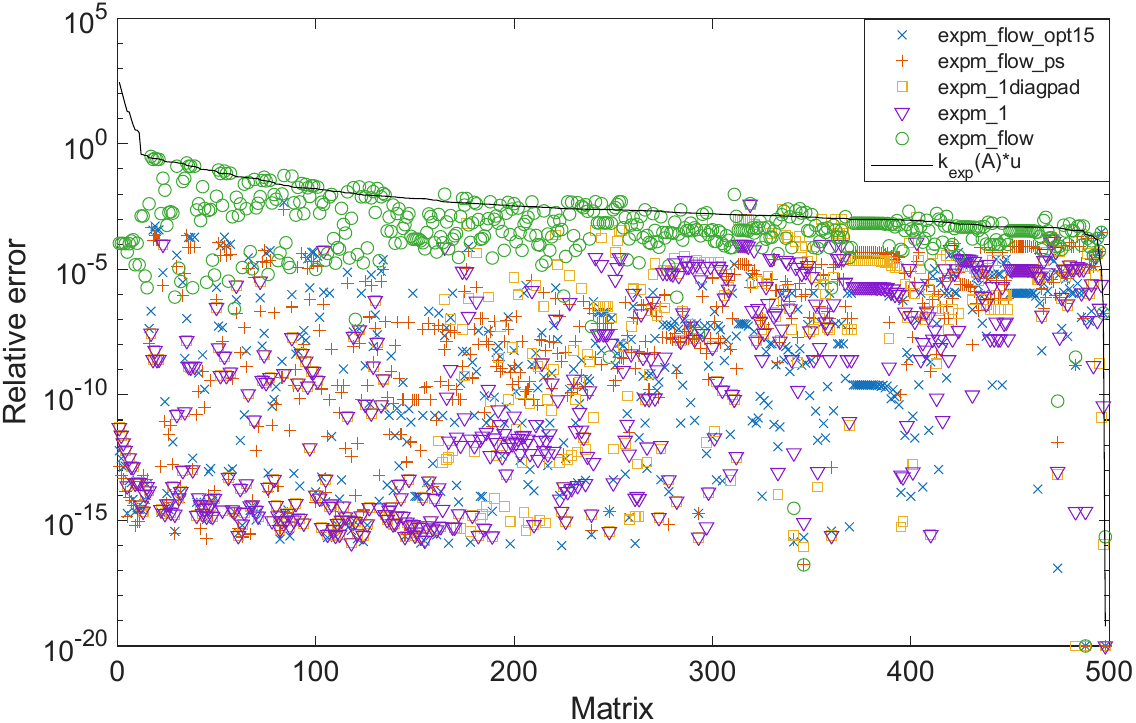}}
        \caption{\footnotesize Normwise relative error ($u=2^{-11}$).}
        \label{fig_set1_a} 
    \end{subfigure} 
    \begin{subfigure}[b]{0.57\textwidth}
        \centerline{\includegraphics[scale=0.35]{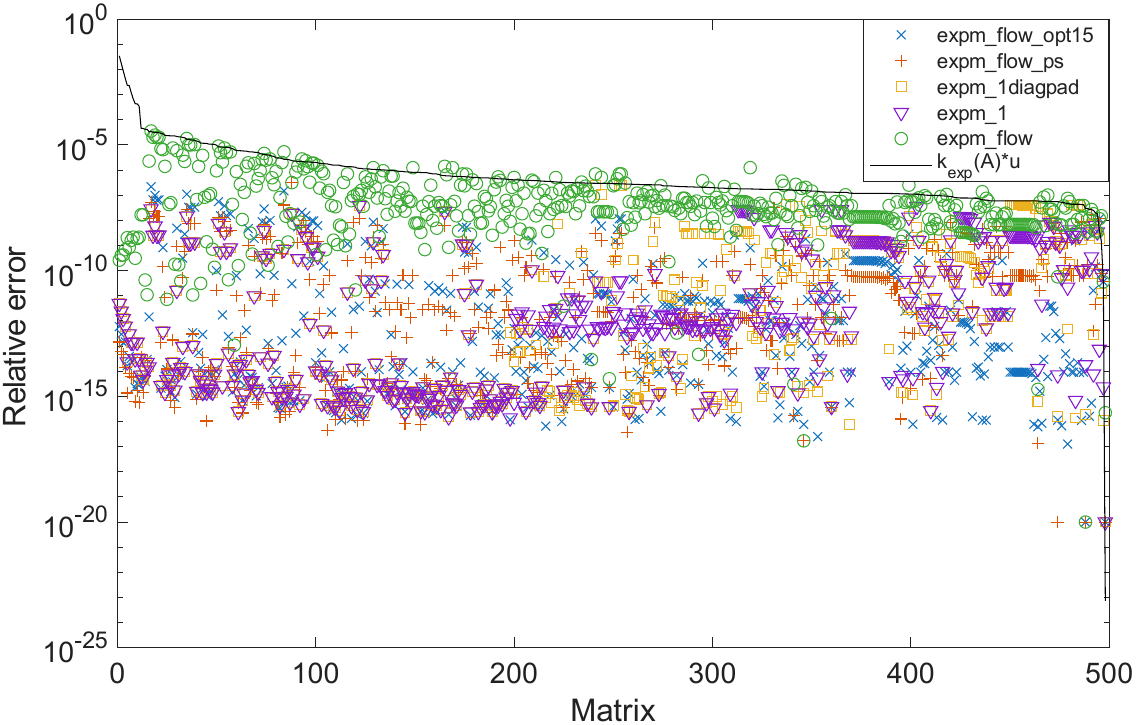}}
        \caption{\footnotesize Normwise relative error ($u=2^{-24}$).}
        \label{fig_set1_b}
    \end{subfigure} \\
    \begin{subfigure}[b]{0.48\textwidth}
        \centerline{\includegraphics[scale=0.36]{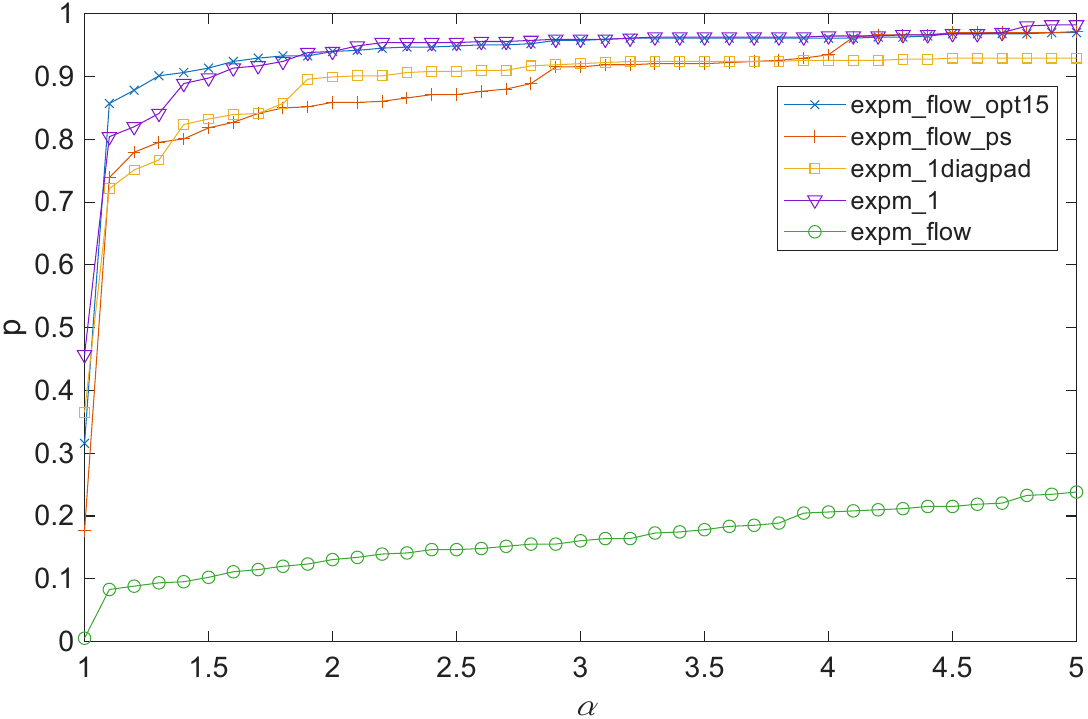}}
        \caption{\footnotesize Performance profile ($u=2^{-11}$).}
        \label{fig_set1_c}
    \end{subfigure}
    \begin{subfigure}[b]{0.55\textwidth}
        \centerline{\includegraphics[scale=0.36]{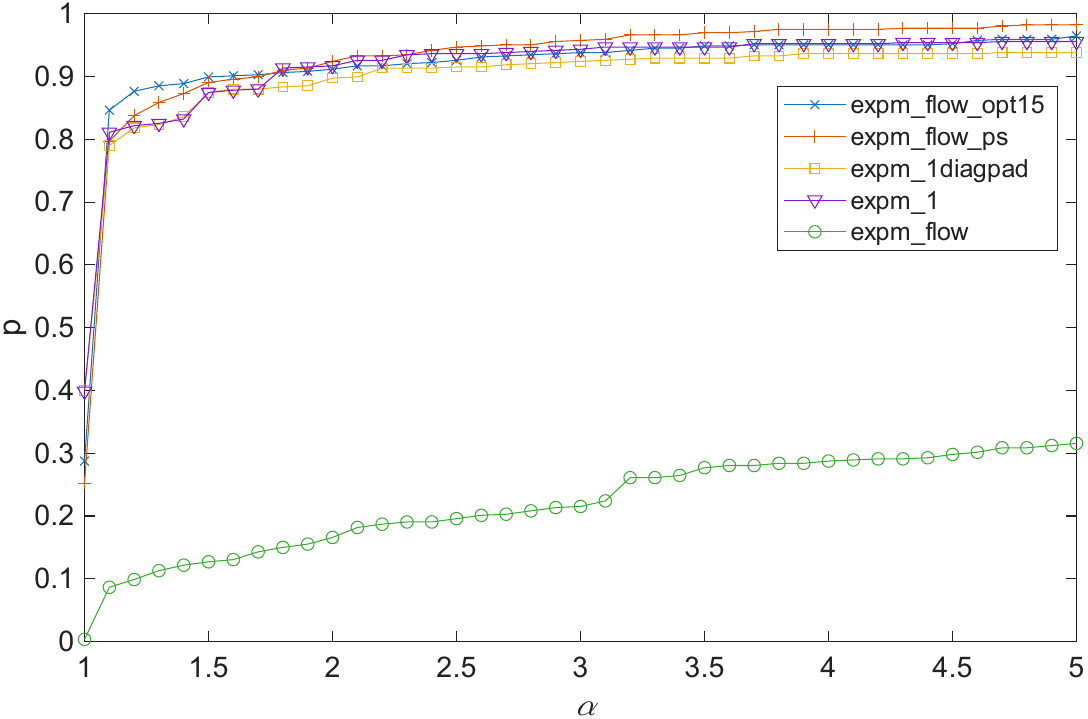}}
        \caption{\footnotesize Performance profile ($u=2^{-24}$).}
        \label{fig_set1_d}
    \end{subfigure} \\
    \begin{subfigure}[b]{0.48\textwidth}
        \centerline{\includegraphics[scale=0.35]{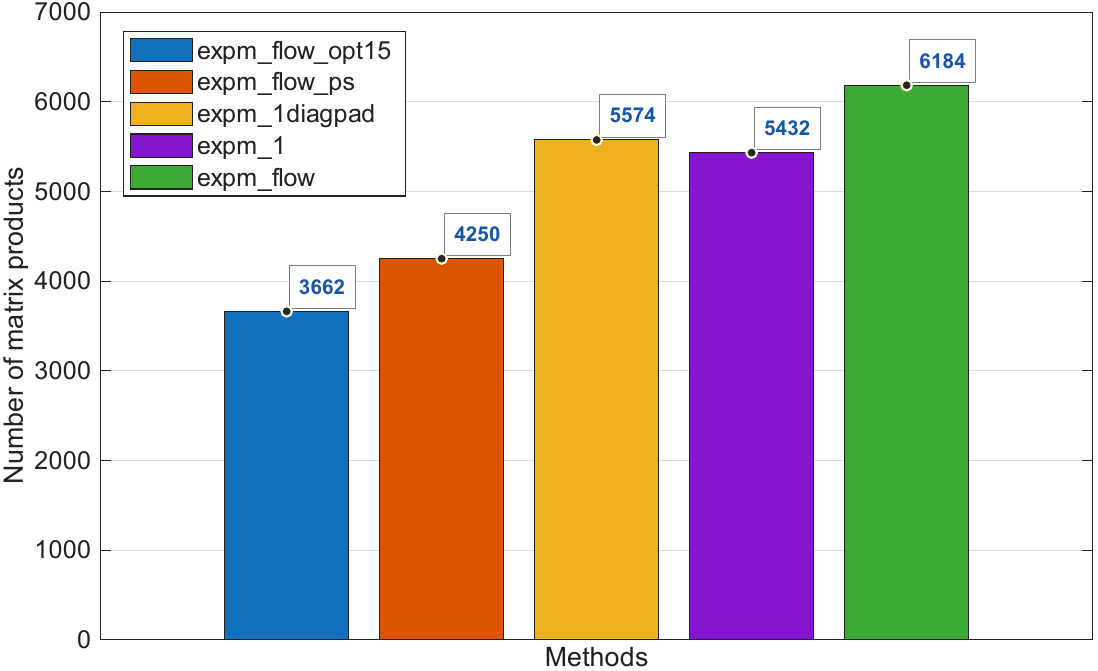}}
        \caption{\footnotesize Number of matrix products ($u=2^{-11}$).}
        \label{fig_set1_e}
    \end{subfigure}
    \begin{subfigure}[b]{0.55\textwidth}
        \centerline{\includegraphics[scale=0.35]{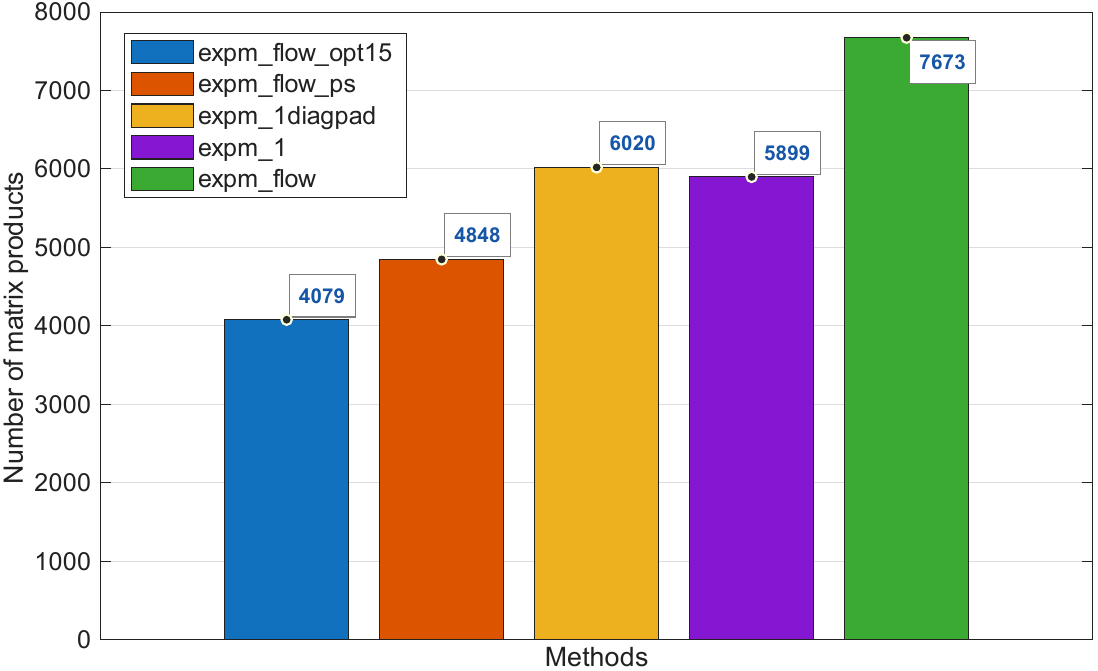}}
        \caption{\footnotesize Number of matrix products ($u=2^{-24}$).}
        \label{fig_set1_f}
    \end{subfigure}
    \begin{subfigure}[b]{0.48\textwidth}
        \centerline{\includegraphics[scale=0.35]{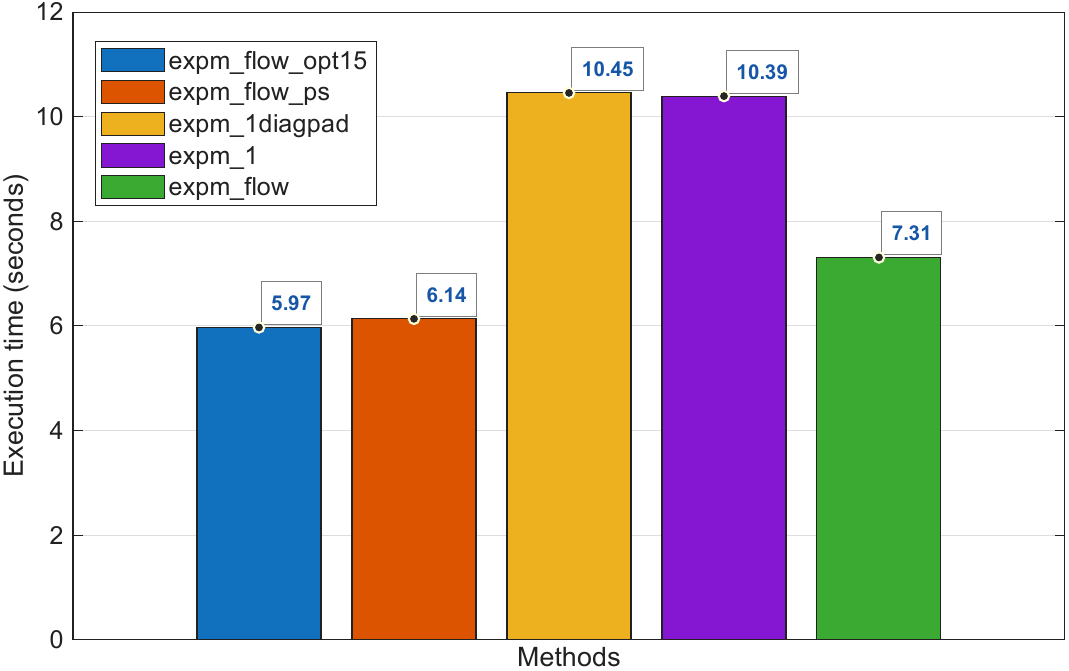}}
        \caption{\footnotesize Execution time ($u=2^{-11}$).}
        \label{fig_set1_g}
    \end{subfigure}
    \begin{subfigure}[b]{0.55\textwidth}
        \centerline{\includegraphics[scale=0.35]{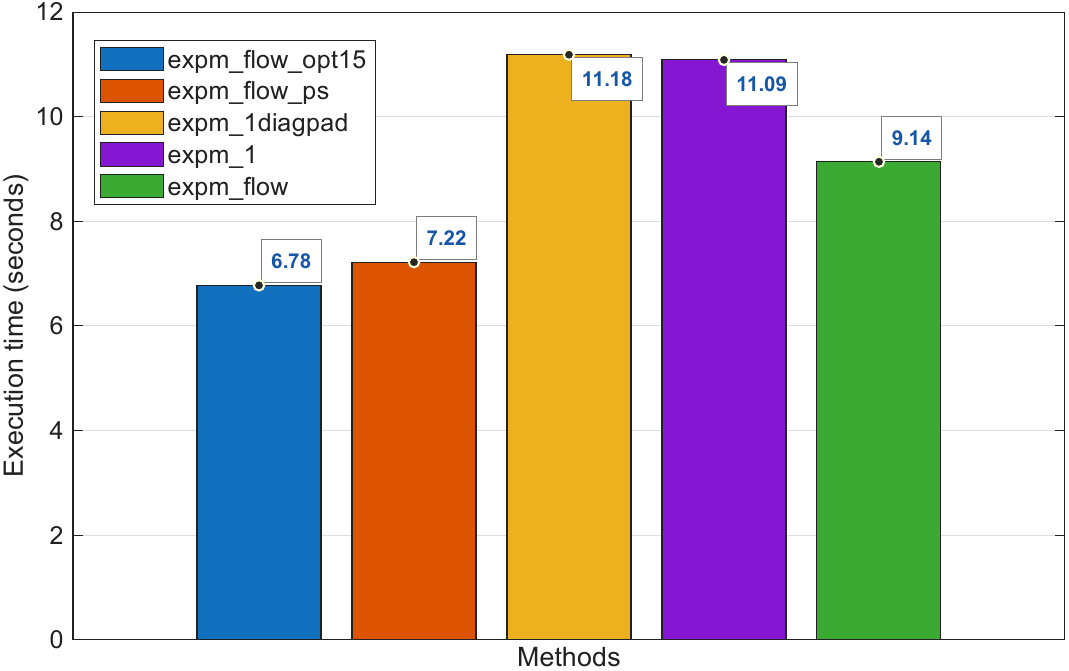}}
        \caption{\footnotesize Execution time ($u=2^{-24}$).}
        \label{fig_set1_h}
    \end{subfigure}
    \caption{Experimental results for matrices in MCT and EMP sets.}
    \label{fig_set1}
\end{figure}

\subsection{Performance evaluation in the PyTorch generative flow environment}\label{sec:numericalresultspytorch}

This subsection presents comparative tests conducted within the generative flow framework of \cite{xiao2020generative}. Matrix exponentials were computed using PyTorch implementations of our proposed functions, \texttt{expm\_flow\_opt15} and \texttt{expm\_flow\_ps}, alongside the original baseline \texttt{expm\_flow}. To ensure consistency with this baseline, both proposed variants employ the matrix $\infty$-norm.

Note that in \texttt{expm\_flow\_opt15}, the
linear combinations of matrices in \eqref{m8s2by1} and
\eqref{m16s2by1}--\eqref{m16s2by2}, and also in \texttt{expm\_flow\_ps}, are accumulated in place, without creating
a new matrix for each term. This does not change the result, the selected
$(m,s)$ or the number of products, but it avoids temporary matrices,
per evaluation of $y_{22}$ or via the P-S technique, whose cost is significant for large inputs. 

All computations were executed using NVIDIA CUDA on a dedicated GPU, following the same hardware environment described in the previous subsection. The proposed method was evaluated during the training process using the CIFAR-10 \cite{krizhevsky2009learning}, ImageNet32, and ImageNet64 \cite{pmlr-v48-oord16} datasets. To obtain representative performance metrics, we analyzed 5000 invocations of the matrix exponential function for each dataset. For each invocation, we recorded metadata for the matrices comprising the input tensor, specifically the number of matrices in the batch, their dimensions, and the maximum $\infty$-norm across the set.

The $\infty$-norms of the matrices studied range from $2.84 \times 10^{-4}$ to $12.57$ for CIFAR-10, from $1.17 \times 10^{-5}$ to $12.49$ for ImageNet32, and from $1.27 \times 10^{-5}$ to $12.8$ for ImageNet64. 

For a normal matrix with the minimum observed norm of $1.27 \times 10^{-5}$, the baseline Algorithm~\ref{Alg_expAIorig} requires a Taylor polynomial of degree $m=1$ with a cost of $1M$ matrix products according to \eqref{C_orig}. In contrast, Algorithm~\ref{Algorithm4} utilizes a degree $m=1$ with a cost of $0M$. 

Similarly, for a normal matrix with the maximum observed norm of $12.8$, the baseline Algorithm~\ref{Alg_expAIorig} selects $m=5$ and $s=5$, leading to a total cost of $10M$ matrix products. Alternatively, Algorithm~\ref{Algorithm4} requires $m=15$ with $s=2$, which reduces the cost to $6M$. This implies that the baseline method is theoretically $1.67$ times more expensive than the proposed approach for this case. 

As an intermediate case, consider a normal matrix with a norm of $0.4$. The reference Algorithm~\ref{Alg_expAIorig} demands a Taylor polynomial of degree $m=7$, at a cost of $7M$ matrix products. However, Algorithm~\ref{Algorithm4} needs a polynomial of degree $m=8$ while requiring only $3M$ matrix products. Since the scaling parameter is $s=0$ for both, the baseline method theoretically requires more than twice as many matrix products as the proposed approach.

In addition, the following metrics were logged for each of the three comparison functions: the degree of the polynomial used, the scaling parameter, the maximum error observed across all matrices in the tensor, the total number of matrix products involved, and the required response time.

When estimating the normwise relative error in the computation of each matrix in (\ref{eq:normwise_relative_error}), the PyTorch implementation of the matrix exponential function, \texttt{torch.linalg.matrix\_exp}, evaluated in double precision (FP64), was taken as the exact reference solution.

Figure \ref{fig_set2} graphically illustrates the results obtained for CIFAR-10 (Figures \ref{fig_set2_a} to \ref{fig_set2_h}), and for ImageNet32 and ImageNet64 (Figures \ref{fig_set2_g} and \ref{fig_set2_h}). For the sake of brevity, the results for ImageNet32 and ImageNet64
are not presented in graph form, as displayed in \ref{fig_set2_a} to \ref{fig_set2_f}, since they are very similar to those for CIFAR-10 and the same conclusions can be drawn. 

Figure \ref{fig_set2_a} plots the normwise relative errors just for the first 150 runs once the training for CIFAR-10 has started. For clarity, these same results are shown sorted, but now for the first 5000 training runs, from largest to smallest error independently for the three functions analyzed in Figure \ref{fig_set2_b}. It should be clarified that there is here no correspondence between a matrix on the $x$-axis and the three errors that appear on the $y$-axis, due to the aforementioned sorting. As we can observe, \texttt{expm\_flow} always resulted in higher errors in the exponential approximation. In contrast, \texttt{expm\_flow\_opt15} and \texttt{expm\_flow\_ps} provided very similar results, with \texttt{expm\_flow\_opt15} being slightly more accurate for the three datasets. This conclusion can also be deduced from the Figure \ref{fig_set2_c} related to the performance profile, where \texttt{expm\_flow\_opt15} and \texttt{expm\_flow\_ps} quickly provided probabilities equal to 1.

Figure \ref{fig_set2_d} also confirms that code \texttt{expm\_flow\_opt15} delivered a more accurate result than the other functions in a higher percentage of cases, with values of 51\%, 49\%, and 50\% of the matrices for the three datasets, respectively. In turn, function \texttt{expm\_flow\_ps} provided similar but slightly lower percentages, with values of 48\%, for CIFAR-10, and 45\%, for ImageNet32 and ImageNet64.

Additionally,  Figure \ref{fig_set2_e} also reveals that \texttt{expm\_flow} used polynomials of a lower order than those of the other functions for CIFAR-10. In contrast, \texttt{expm\_flow} required higher values of the scaling parameter, as can be appreciated in Figures \ref{fig_set2_f}. Identical behaviour is exhibited by the three methods for ImageNet32 and ImageNet64 datasets. 

Finally, regarding the number of products, Figure \ref{fig_set2_g} states that \texttt{expm\_flow} carried out 1.99, 1.88, and 1.88 times more products than \texttt{expm\_flow\_opt15} in the calculation of matrix exponentials for the three datasets, respectively. Similarly, Figure \ref{fig_set2_h} shows that the exponential computation based on function \texttt{expm\_flow\_opt15} was, for each dataset, 2.52, 2.27, and 2.86 times faster than the same calculation using function \texttt{expm\_flow}.

\begin{figure}[H]
    \begin{subfigure}[b]{0.48\textwidth}
        \centerline{\includegraphics[scale=0.34]{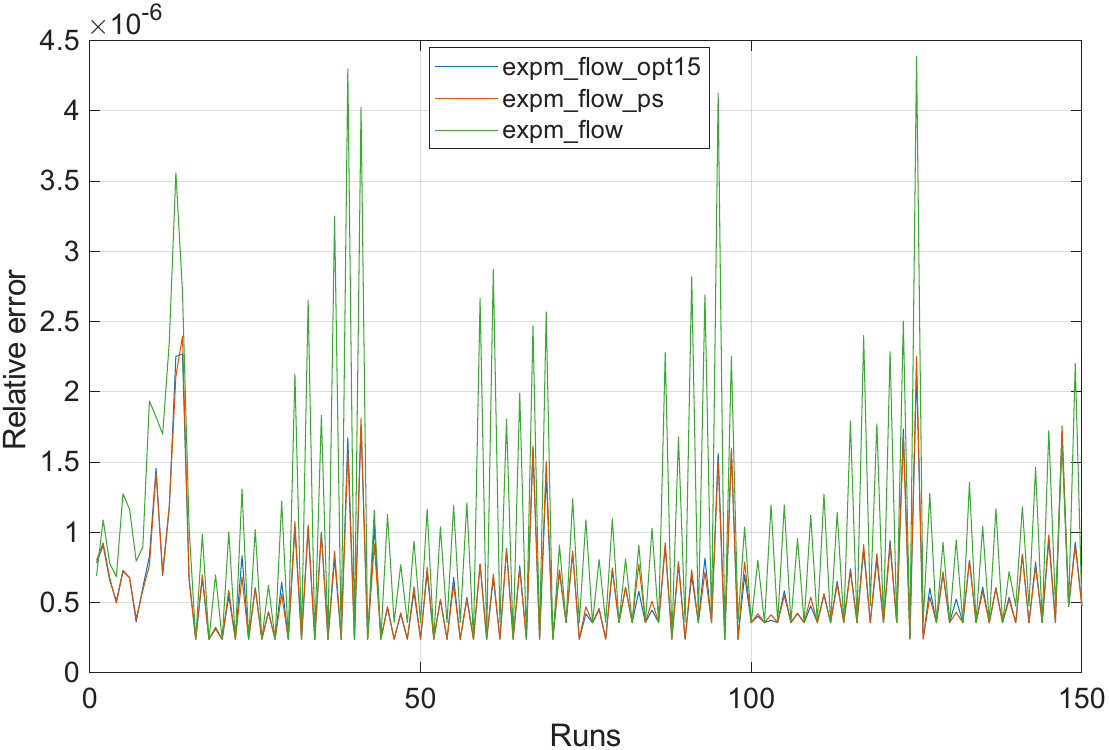}}
        \caption{\footnotesize Normwise relative error.}
        \label{fig_set2_a} 
    \end{subfigure} 
    \begin{subfigure}[b]{0.55\textwidth}
        \centerline{\includegraphics[scale=0.34]{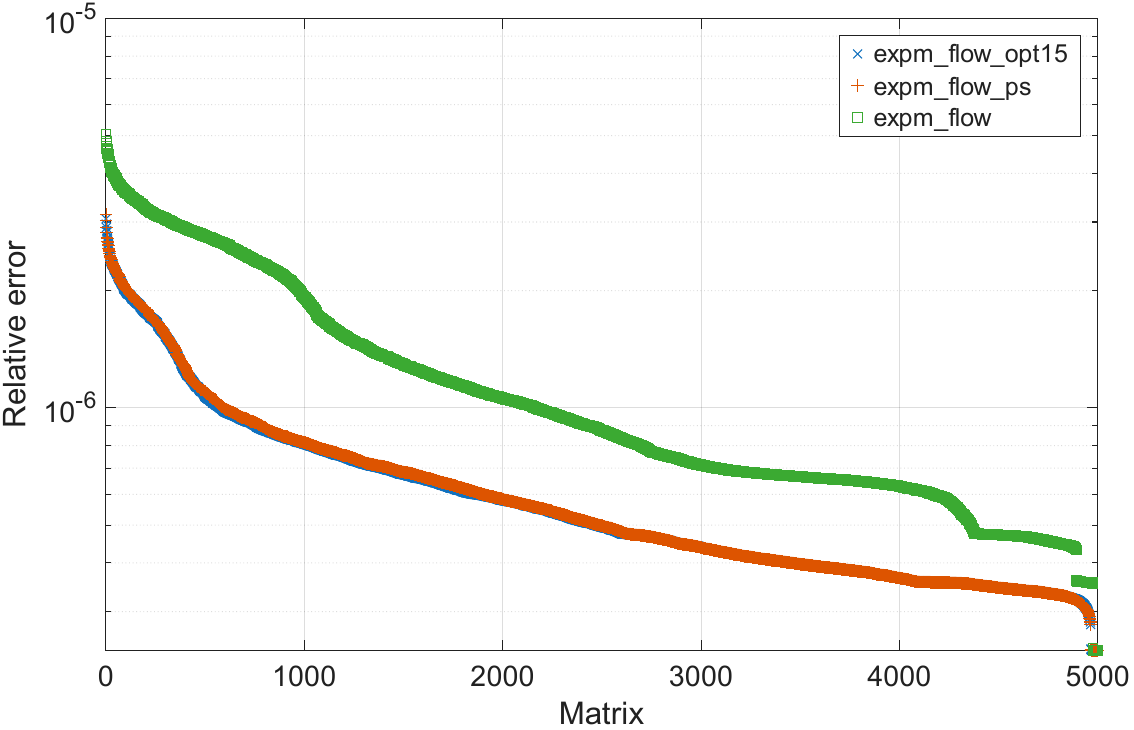}}
        \caption{\footnotesize Unmatched ordered normwise relative error.}
        \label{fig_set2_b}
    \end{subfigure} \\
    \begin{subfigure}[b]{0.48\textwidth}
        \centerline{\includegraphics[scale=0.34]{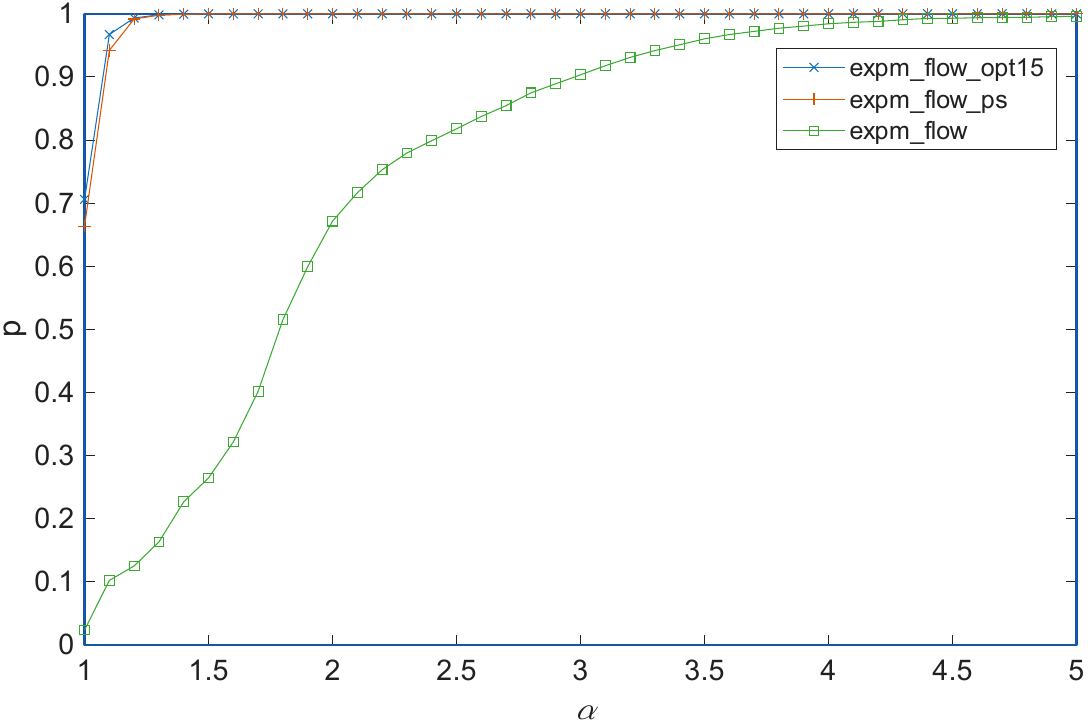}}
        \caption{\footnotesize Performance profile.}
        \label{fig_set2_c}
    \end{subfigure}
    \begin{subfigure}[b]{0.58\textwidth}
        \centerline{\fbox{\raisebox{0.65cm}{\includegraphics[scale=0.34]{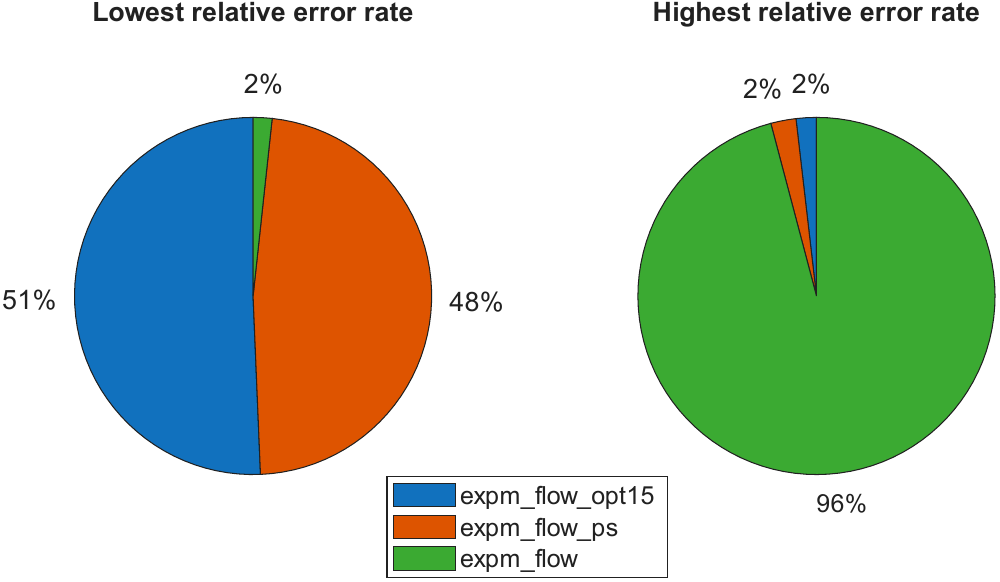}}}}
        \caption{\footnotesize Lowest and highest relative error rate.}
        \label{fig_set2_d}
    \end{subfigure} \\
    \begin{subfigure}[b]{0.48\textwidth}
        \centerline{\includegraphics[scale=0.34]{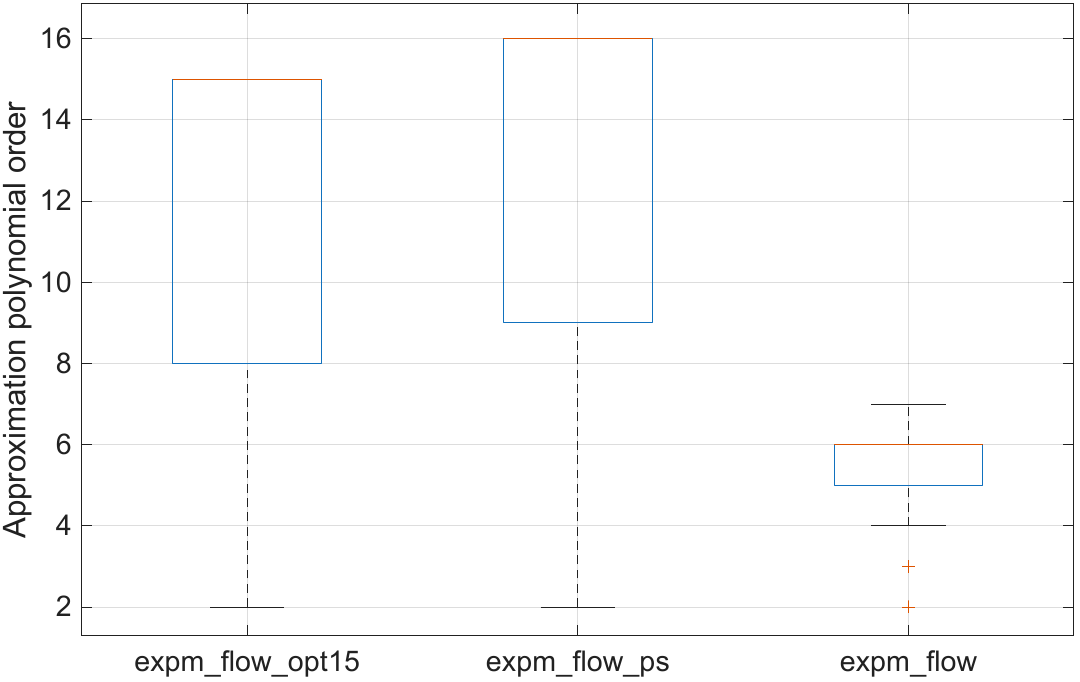}}
        \caption{\footnotesize Polynomial order.}
        \label{fig_set2_e}
    \end{subfigure}
    \begin{subfigure}[b]{0.48\textwidth}
        \centerline{\includegraphics[scale=0.34]{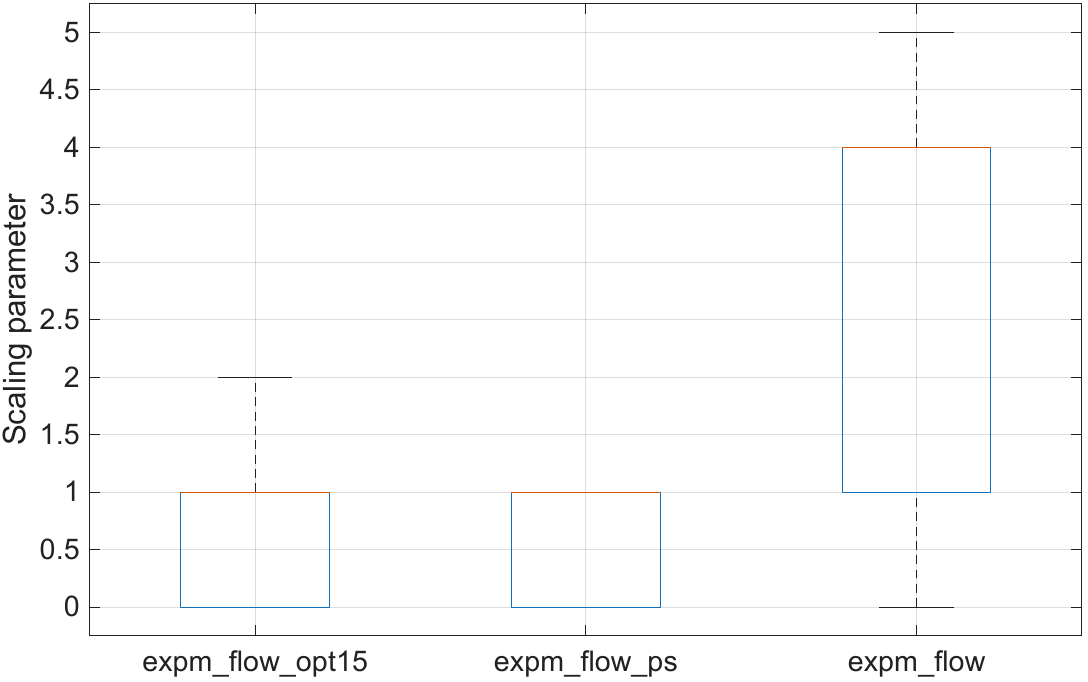}}
        \caption{\footnotesize Scaling parameter.}
        \label{fig_set2_f}
    \end{subfigure}
    \begin{subfigure}[b]{0.48\textwidth}
        \centerline{\includegraphics[scale=0.34]{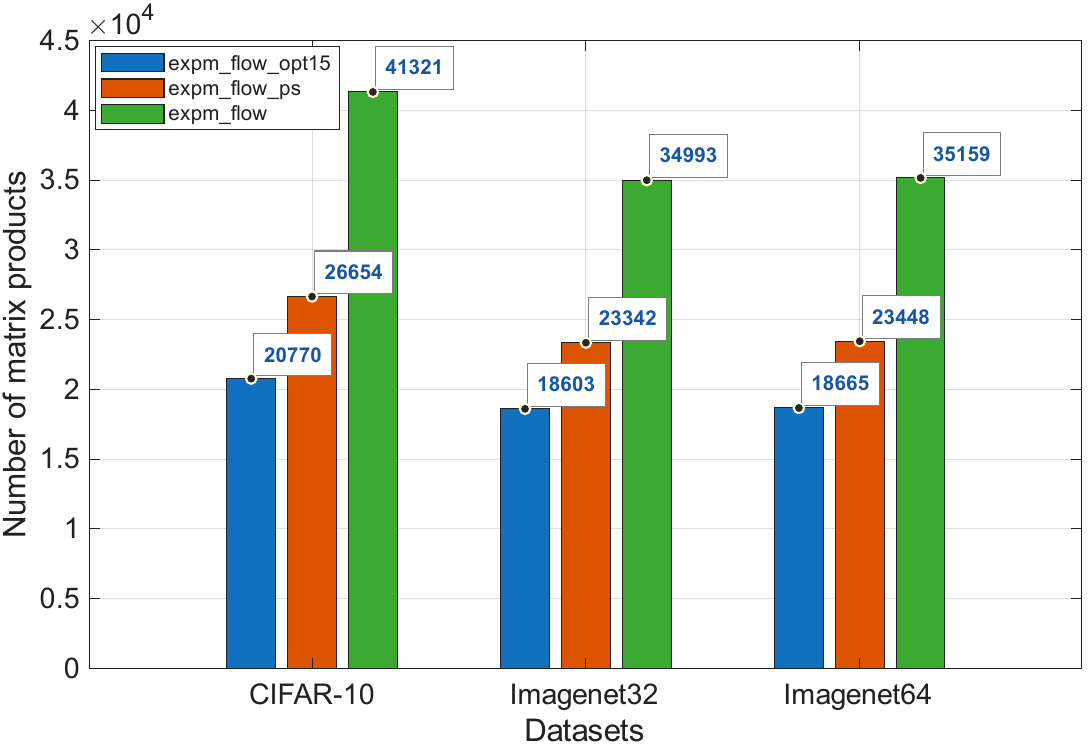}}
        \caption{\footnotesize Number of matrix products.}
        \label{fig_set2_g}
    \end{subfigure}
    \begin{subfigure}[b]{0.55\textwidth}
        \centerline{\includegraphics[scale=0.34]{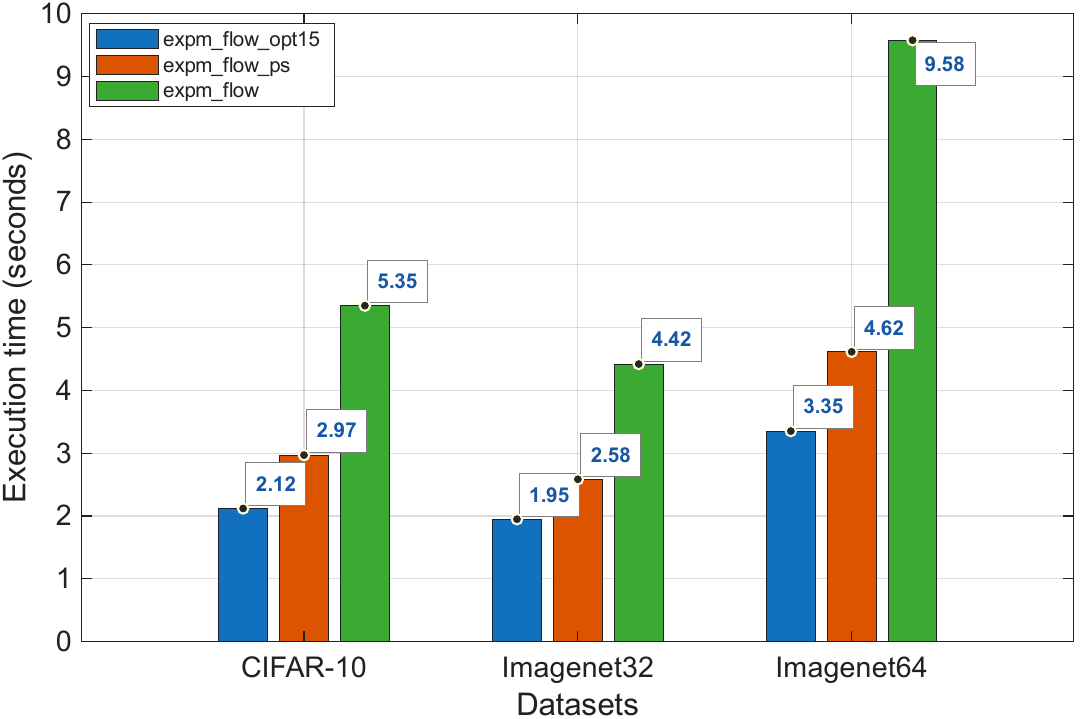}}
        \caption{\footnotesize Execution time.}
        \label{fig_set2_h}
    \end{subfigure}
    \caption{Experimental results for the datasets.}
    \label{fig_set2}
\end{figure}

\subsection{Comparison with \texttt{torch.linalg.matrix\_exp} and \texttt{scipy.linalg.expm}}\label{sec:comparisonlibraries}

To compare the proposed method with state-of-the-art implementations in a
PyTorch/GPU setting, we evaluated \texttt{expm\_flow\_opt15} against the
PyTorch routine \texttt{torch.linalg.matrix\_exp}, which implements the
optimized Taylor method of Bader, Blanes and Casas~\cite{BBS19}, and against
\path{scipy.linalg.expm}, which implements the Pad\'e-based scaling and
squaring algorithm of Al-Mohy and Higham~\cite{AlHi09}. All methods were run
in single precision (FP32) on matrices $A=0.1\,R$, where $R\in\mathbb{R}^{n\times n}$
has independent standard normal entries (random seed $n$). The relative error
was measured in the Frobenius norm against \texttt{torch.linalg.matrix\_exp}
evaluated in double precision (FP64). \texttt{expm\_flow\_opt15} was run with the
FP32 unit roundoff $\varepsilon=2^{-24}$ and with $\varepsilon=10^{-8}$.
\texttt{expm\_flow\_opt15} and \texttt{torch.linalg.matrix\_exp} were timed on
an NVIDIA L40S GPU (best of 10 runs), while \texttt{scipy.linalg.expm} executes on the CPU and is therefore included solely for accuracy comparisons, as its execution times were significantly higher than those of the other algorithms.

\begin{table}[H]
\begin{center}
\resizebox{\textwidth}{!}{%
{\footnotesize\setlength{\tabcolsep}{3pt}
\begin{tabular}{rr|cccc|cc|cc}
\hline
 & & \multicolumn{4}{c|}{Relative error} & \multicolumn{2}{c|}{Products, opt15} & \multicolumn{2}{c}{Time (ms)}\\
$n$ & $\|A\|_1$ & opt15, $2^{-24}$ & opt15, $10^{-8}$ & torch & scipy & $2^{-24}$ & $10^{-8}$ & opt15 & torch\\
\hline
16 & 1.85 & $2.5\times10^{-8}$ & $2.5\times10^{-8}$ & $9.9\times10^{-8}$ & $6.2\times10^{-8}$ & 4 & 4 & 0.14 & 0.34\\
64 & 6.14 & $5.7\times10^{-8}$ & $1.7\times10^{-7}$ & $6.6\times10^{-7}$ & $1.4\times10^{-7}$ & 4 & 5 & 0.15 & 0.37\\
128 & 11.69 & $1.2\times10^{-7}$ & $1.2\times10^{-7}$ & $5.2\times10^{-7}$ & $2.4\times10^{-7}$ & 5 & 5 & 0.16 & 0.35\\
256 & 22.97 & $1.5\times10^{-7}$ & $1.5\times10^{-7}$ & $7.3\times10^{-7}$ & $4.6\times10^{-7}$ & 6 & 6 & 0.18 & 0.38\\
512 & 44.82 & $5.0\times10^{-7}$ & $5.0\times10^{-7}$ & $2.5\times10^{-6}$ & $8.9\times10^{-7}$ & 7 & 7 & 0.20 & 0.40\\
1024 & 87.05 & $4.0\times10^{-7}$ & $4.0\times10^{-7}$ & $4.3\times10^{-6}$ & $2.9\times10^{-6}$ & 7 & 7 & 0.63 & 1.08\\
2048 & 172.07 & $1.3\times10^{-6}$ & $1.3\times10^{-6}$ & $1.4\times10^{-5}$ & $5.0\times10^{-5}$ & 8 & 8 & 3.32 & 5.69\\
4096 & 340.03 & $3.7\times10^{-6}$ & $3.7\times10^{-6}$ & $3.4\times10^{-5}$ & $3.8\times10^{-4}$ & 9 & 9 & 33.56 & 45.48\\
\hline
\end{tabular}}}
\end{center}
\caption{Relative error, number of matrix products and GPU execution time of
\texttt{expm\_flow\_opt15} (opt15), \texttt{torch.linalg.matrix\_exp}
(torch,~\cite{BBS19}) and \texttt{scipy.linalg.expm} (scipy,~\cite{AlHi09}) in
single precision. Execution times correspond to $\varepsilon=2^{-24}$ for
opt15.}\label{tab:comparisonlibraries}
\end{table}

Table~\ref{tab:comparisonlibraries} shows that \texttt{expm\_flow\_opt15}
attains the smallest error for every size, between about $4$ and $12$ times
smaller than \texttt{torch.linalg.matrix\_exp}, using between 4 and 9 matrix
products.
For $n\geq 2048$ the error of \path{scipy.linalg.expm} in single precision
grows faster than that of the other two methods, reaching
$3.8\times10^{-4}$ for $n=4096$, about $100$ times larger than that of
\texttt{expm\_flow\_opt15}. On the GPU, \texttt{expm\_flow\_opt15} was also
between $1.36$ and $2.55$ times faster than \texttt{torch.linalg.matrix\_exp}
for all the sizes considered. Setting $\varepsilon=10^{-8}$ instead of
$2^{-24}$ changed the number of products only for $n=64$.

\subsection{Performance analysis in generative flow models}\label{sec:performanceanalysis}

To evaluate the efficiency gains of the proposed matrix exponential approach (\texttt{expm\_flow\_opt15}) relative to the baseline (\texttt{expm\_flow}), we conducted a series of numerical experiments on the three aforementioned datasets: CIFAR-10, ImageNet32, and ImageNet64. Unless stated otherwise, the experiments in this section were run on NVIDIA L40S GPUs with PyTorch 2.14 and CUDA 13.

In the generative model framework of \cite{xiao2020generative}, the reduced-rank approximation \eqref{eq:expmWV} is used only in coupling layers with more than 24 channels. With the three levels of the model used in this section, the coupling layers have 6, 12 and 24 channels for the three datasets, so all the matrix exponentials, of orders 6 to 48, are computed with the full matrix exponential: those of the coupling layers, of orders 6, 12 and 24, and those of the invertible $1\times1$ convolutions, of orders 12, 24 and 48. We verified that the approximation \eqref{eq:expmWV} was not called in any of the runs of this section.

The evaluation consists of two parts. First, isolated benchmarks were performed to compare the execution times of both methods across various matrix sizes, independent of other neural network components. Second, both algorithms were integrated into the coupling layers of a Glow-based architecture \cite{kingma2018glow}, following the methodology described in \cite{xiao2020generative}. Performance is quantified by measuring total computation time and calculating the speed-up ratio, defined as the execution time of \texttt{expm\_flow} divided by that of \texttt{expm\_flow\_opt15}.

The flow model architecture, illustrated in Figure \ref{diagram}, follows a multi-scale structure. The squeeze operation performs downsampling by halving the spatial dimensions and quadrupling the channel count. Subsequently, the split operation partitions the feature maps along the channel dimension; one portion is directed to the output as a latent representation, while the remaining part is propagated through the subsequent layers of the model. As in the code of \cite{xiao2020generative}, the models were trained with the Adamax optimizer~\cite{kingma2014adam} with a learning rate of 0.01, one warm-up epoch, and the learning rate halved after 40 epochs for CIFAR-10 and after 20 and 40 epochs for ImageNet32 and ImageNet64.

\begin{figure}[H]
\includegraphics[width=\textwidth]{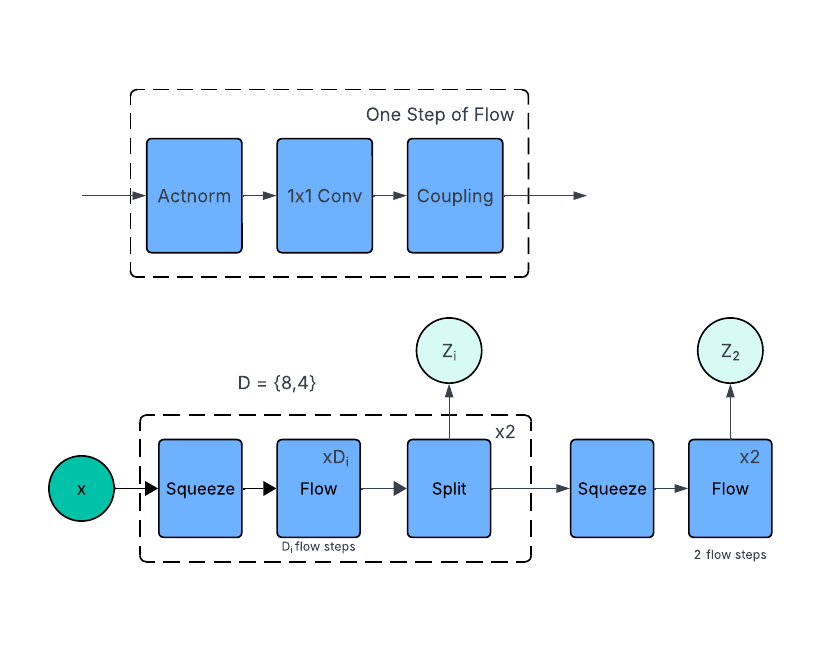}
\caption{Diagram overview of flow model trained to generate images. } \label{diagram}
\end{figure}

\subsubsection{Computational scaling and matrix dimensionality}\label{sec:weightsize}

This analysis evaluates how the execution time of the matrix exponential
scales with the matrix dimension and with the number of matrices in a batch.

Figure~\ref{weightsize} compares, on a log--log scale, the execution time of \texttt{expm\_flow}, \texttt{expm\_flow\_opt15}, \texttt{expm\_flow\_opt15\_fixed} (the middle-ground baseline using a fixed approximation order of $m=15$ and the scaling parameter from \eqref{eq:sfixed} defined above), and \texttt{torch.linalg.matrix\_exp}~\cite{BBS19}. All tests were run in single precision. The left plot shows performance for a single $n \times n$ matrix with $n = 2, \ldots, 1024$, while the right plot shows performance for a batch of $N$ matrices of size $1024 \times 1024$ for $N = 1, \ldots, 1024$. The matrices are $0.1\,R$, with real matrices $R$ having independent standard normal entries;
for $1024\times1024$ matrices this gives $\|A\|_1\approx 88$, for which
\texttt{expm\_flow\_opt15} uses $m=15+$ and $s=3$ (7 products) and
\texttt{expm\_flow} uses 12 products. The left panel was timed on an Intel CPU
with 8 threads, where the $\mathcal{O}(n^3)$ growth is visible; on the GPU a
single matrix of size $n\leq512$ takes about $0.1$--$0.4$~ms for every method,
since the time is dominated by kernel launches. The right panel was timed on
an NVIDIA L40S GPU. Times are the best of 5 (CPU) or 10 (GPU) runs.


\begin{figure}[H]
\centering
\includegraphics[width=\textwidth]{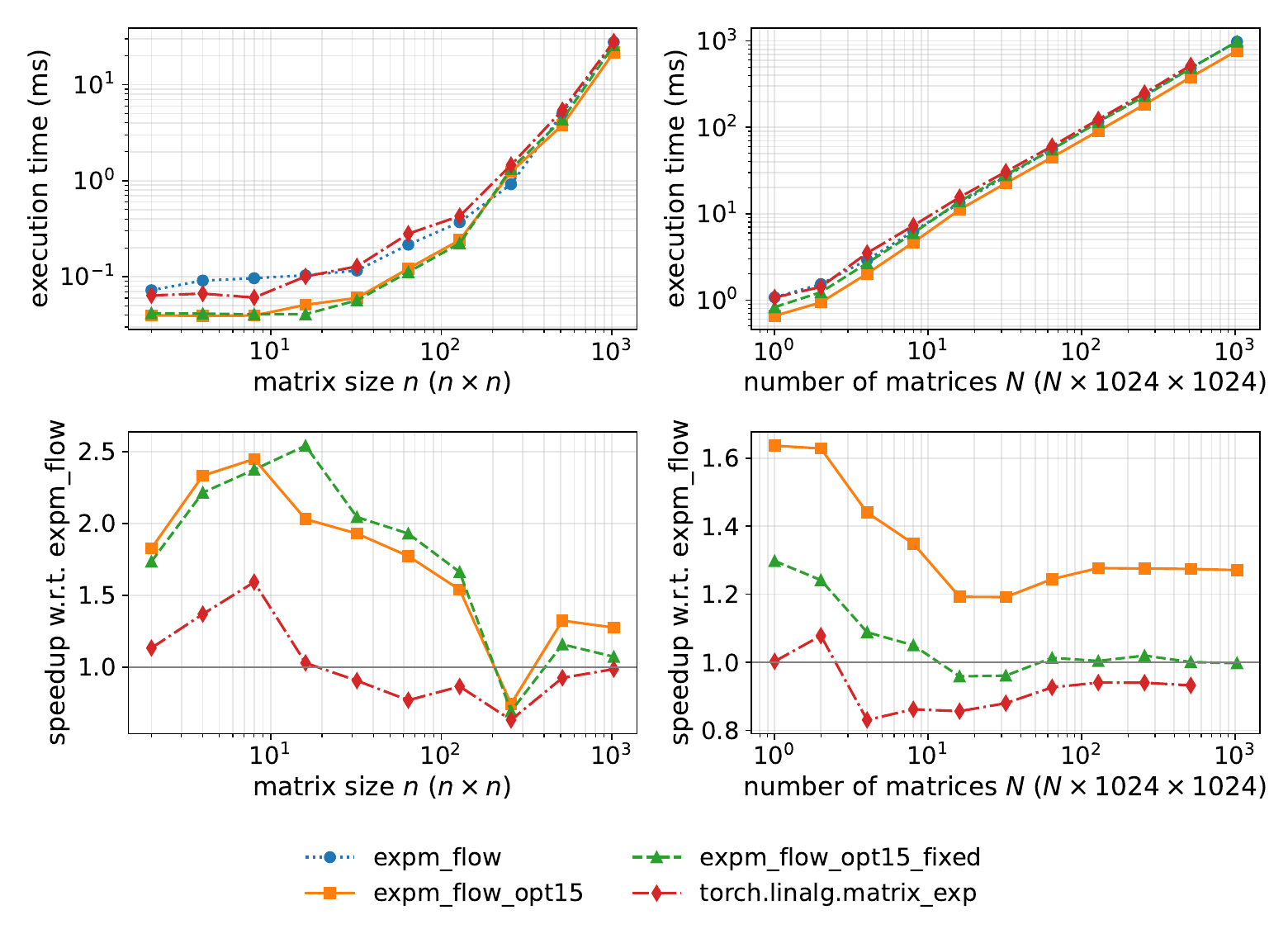}
\caption{Execution time (top) and speedup with respect to
\texttt{expm\_flow} (bottom) of the matrix exponential in single precision.
(Left) A single $n\times n$ matrix, Intel CPU. (Right) A batch of $N$
matrices of size $1024\times1024$, NVIDIA L40S GPU;
\texttt{torch.linalg.matrix\_exp} ran out of GPU memory for $N=1024$.}\label{weightsize}
\end{figure}

For a single matrix (left), \texttt{expm\_flow\_opt15} is between $1.5$ and
$2.5$ times faster than \texttt{expm\_flow} for $n\leq128$, and $1.32$ and
$1.27$ times faster for $n=512$ and $n=1024$. For $n=256$,
\texttt{expm\_flow} was faster than all the other methods, including
\texttt{torch.linalg.matrix\_exp}. This is not due to the algorithms: for the matrix of Figure~\ref{weightsize} and ten other random matrices of this size, \texttt{expm\_flow} uses $m=3$ and $s=6$ (10 products), and \texttt{expm\_flow\_opt15} uses $m=15+$ and $s=2$ (6 products). At this size on the CPU, the matrix products take less than half of the execution time, and the time of the remaining $\mathcal{O}(n^2)$ operations varied between runs: in three repetitions of the experiment on the same node, \texttt{expm\_flow\_opt15} took $0.66$, $1.15$ and $1.10$~ms, and \texttt{expm\_flow} $1.02$, $0.91$ and $0.88$~ms. 

For a batch of $1024\times1024$ matrices
(right), \texttt{expm\_flow\_opt15} is the fastest method. It is between $1.19$ and $1.64$ times faster
than \texttt{expm\_flow} and between $1.34$ and $1.73$ times faster than
\texttt{torch.linalg.matrix\_exp}, with a speedup of about $1.27$ with respect to
\texttt{expm\_flow} for $N\geq128$. 

Compared with the middle-ground baseline,
\texttt{expm\_flow\_opt15} is between $1.23$ and $1.32$ times faster on the GPU,
since the dynamic selection uses fewer squarings ($s=3$ instead of $s=6$); on
the CPU, for $16\leq n\leq128$, \texttt{expm\_flow\_opt15\_fixed} is
slightly faster although it uses up to two more squarings, because at these
sizes the norm computations of the selection of $(m,s)$ cost about as much as
a matrix product.

These wall-clock speedups are smaller than the ratio of matrix products
($12/7\approx1.71$ for $n=1024$). The evaluation also involves
$\mathcal{O}(n^2)$ operations, namely matrix additions, scalings and norm
computations, and on current hardware a matrix product is fast enough that
these operations take a significant share of the time. The number of matrix products
therefore remains the main, hardware-independent measure of cost, while the
wall-clock gain depends on how efficiently the remaining operations are
implemented.

\subsubsection{Middle-ground baseline and dependence on the tolerance}\label{sec:middleground}

In this subsection, we compare the computational costs and dependence on the tolerance $\varepsilon$ across three implementations: \texttt{expm\_flow\_opt15}, the middle-ground baseline \texttt{expm\_flow\_opt15\_fixed}, and \texttt{expm\_flow}.

\paragraph{Cost during training}
Table~\ref{tab:middleground} gives the number of matrix products required by
the three functions during one epoch of the flow model on CIFAR-10,
including the initialization and the two evaluations of the test set after
the epoch, which involves $30{,}716$ calls to the matrix exponential
($21{,}896$ of them in training). For the
middle-ground baseline, $\varepsilon$ in \eqref{eq:sfixed} was set to the
single-precision unit roundoff $u=2^{-24}\approx5.96\times10^{-8}$ used in
training, which is less demanding than the tolerance $\varepsilon=10^{-8}$ of
the other two functions, so this comparison slightly favours the
middle-ground baseline. \texttt{expm\_flow} requires $1.96$ times as many products as
\texttt{expm\_flow\_opt15}, and \texttt{expm\_flow\_opt15\_fixed} requires
$1.14$ times as many. Most of the reduction is therefore due to the
evaluation formulas, and the dynamic $(m,s)$ selection saves a further
$12\%$ of the products with respect to the middle-ground baseline.

\begin{table}[H]
\begin{center}
{\footnotesize
\begin{tabular}{l|ccc}
\hline
Function & Matrix products & Products per call & Ratio to \texttt{expm\_flow\_opt15}\\
\hline
\texttt{expm\_flow}               & 265{,}856 & 8.66 & 1.96\\
\texttt{expm\_flow\_opt15\_fixed} & 154{,}346 & 5.02 & 1.14\\
\texttt{expm\_flow\_opt15}        & 135{,}711 & 4.42 & 1.00\\
\hline
\end{tabular}}
\end{center}
\caption{Matrix products required by each function during one
epoch of the flow model on CIFAR-10, including initialization and testing
($30{,}716$ matrix exponentials).}\label{tab:middleground}
\end{table}

\paragraph{Dependence on the tolerance}
Figure~\ref{fig:tolerancesweep} shows the mean number of matrix products of the
three functions, and the resulting speedup of \texttt{expm\_flow\_opt15},
for tolerances $\varepsilon$ from $10^{-1}$ to $10^{-14}$. Here all three
functions use the same $\varepsilon$: as the stopping tolerance of the Taylor
loop in \texttt{expm\_flow} (Algorithm~\ref{Alg_expAIorig}), in
\eqref{eq:sfixed} for \texttt{expm\_flow\_opt15\_fixed}, and in
Algorithm~\ref{Algorithm4} for \texttt{expm\_flow\_opt15}. The test set
consists of $60$ real matrices $W=0.5\,R$, with $R$ having independent standard
normal entries, of sizes $12$, $24$ and $48$ ($20$ of each size), the sizes of
the exponentiated matrices in the flow model, with $\|W\|_\infty$ between
$5.5$ and $26.6$. The computations were carried out in double precision.

\begin{figure}[H]
\centering
\includegraphics[width=\textwidth]{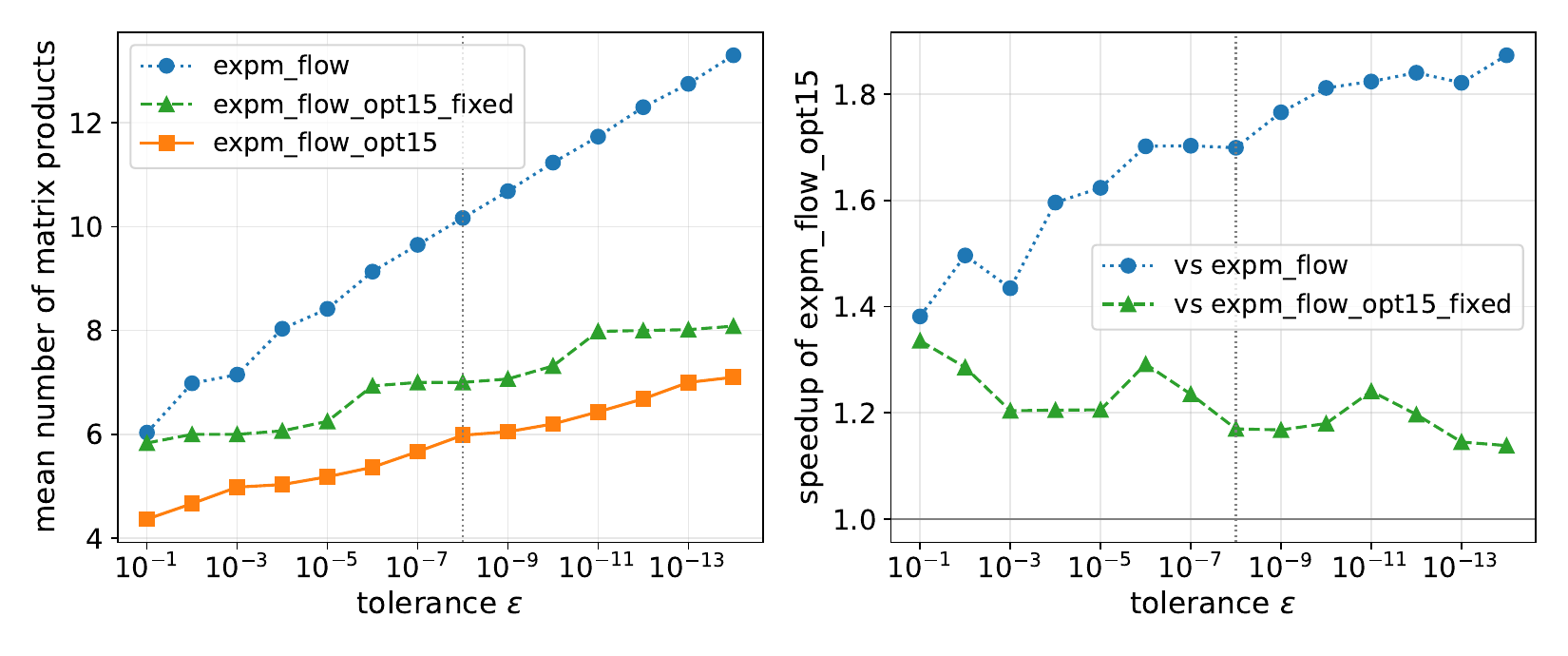}
\caption{Mean number of matrix products (left) and speedup of
\texttt{expm\_flow\_opt15} in number of products (right) as a function of
the tolerance $\varepsilon$, with all functions using the same $\varepsilon$.
The dotted vertical line marks $\varepsilon=10^{-8}$, the value used in the
rest of the paper.}\label{fig:tolerancesweep}
\end{figure}

The speedup with respect to \texttt{expm\_flow} grows as the tolerance
decreases, from $1.38$ for $\varepsilon=10^{-1}$ to $1.70$ for
$\varepsilon=10^{-8}$ and $1.87$ for $\varepsilon=10^{-14}$, since the number
of products of the baseline grows faster as $\varepsilon$ decreases. With
respect to the middle-ground baseline the speedup remains between $1.14$ and
$1.34$ over the whole range, so the dynamic $(m,s)$ selection reduces the
cost for every tolerance considered, by between $12\%$ and $25\%$ of the
products.

\subsubsection{Training Performance and Scaling}
To evaluate the practical utility of our approach, we integrated the proposed matrix exponential method into a generative framework. Glow \cite{kingma2018glow} is a coupling layer approach to learning invertible transformations. In the public code of~\cite{xiao2020generative}, we replaced the original implementation (\texttt{expm\_flow}) with our optimized function (\texttt{expm\_flow\_opt15}), leaving the rest of the training code unchanged. As in Section~\ref{sec:numericalresultspytorch}, we compare the time of the matrix exponentials themselves: every call made during training was timed on the GPU, both the evaluation (forward pass) and the computation of its gradient (backward pass). Table~\ref{tab:expmtraining} gives the results for 10~epochs on CIFAR-10 and 1~epoch on ImageNet32 and ImageNet64, with each function run in a separate job.

\begin{table}[H]
\begin{center}
\resizebox{\textwidth}{!}{%
{\footnotesize\setlength{\tabcolsep}{4pt}
\begin{tabular}{ll|rrrrr}
\hline
 & & Products & Time of & Exponential, & Exponential, & Share of\\
Dataset & Function & per epoch & products (s) & forward (s) & total (s) & epoch\\
\hline
CIFAR-10   & \texttt{expm\_flow}        & 200{,}207 & 8.32 & 18.58 & 38.02 & 23\%\\
           & \texttt{expm\_flow\_opt15} & 101{,}411 (1.97) & 4.78 (1.74) & 11.57 (1.61) & 25.13 (1.51) & 14\%\\
\hline
ImageNet32 & \texttt{expm\_flow}        & 5{,}134{,}832 & 271.3 & 642.8 & 1179.0 & 22\%\\
           & \texttt{expm\_flow\_opt15} & 2{,}592{,}426 (1.98) & 118.1 (2.30) & 275.3 (2.34) & 620.8 (1.90) & 15\%\\
\hline
ImageNet64 & \texttt{expm\_flow}        & 5{,}147{,}683 & 451.0 & 777.9 & 2617.9 & 22\%\\
           & \texttt{expm\_flow\_opt15} & 2{,}586{,}394 (1.99) & 224.0 (2.01) & 399.2 (1.95) & 1462.3 (1.79) & 14\%\\
\hline
\end{tabular}}}
\end{center}
\caption{Matrix exponentials during training on an NVIDIA L40S GPU: matrix products per epoch, time spent in matrix products, time of the matrix exponential (forward pass, and forward plus backward pass), and share of the training time. Speedups with respect to \texttt{expm\_flow} are given in parentheses.}\label{tab:expmtraining}
\end{table}

The proposed function requires about half the matrix products of \path{expm_flow}, and the time spent in matrix products is reduced by a factor between $1.74$ and $2.30$. The matrix exponential, including its backward pass, is between $1.51$ and $1.90$ times faster with \texttt{expm\_flow\_opt15}, consistent with the speedups of Section~\ref{sec:numericalresultspytorch}. The effect on the total training time is smaller: with \texttt{expm\_flow} the matrix exponentials take between $22\%$ and $23\%$ of the training time, and the rest of each epoch, spent in the coupling networks, the optimizer and data loading, is the same for both functions, so even a matrix exponential of negligible cost would make training at most about $1.29$ times faster. Applications in which the matrix exponential takes a larger share of the computation will benefit correspondingly more.

\subsubsection{Inference and Sampling Performance}

While the previous sections evaluated performance during the training phase, efficiency during inference is equally critical for generative models. Table \ref{sampling} presents a comparison of the sampling times between  \texttt{expm\_flow} and \texttt{expm\_flow\_opt15}, for a CIFAR-10 model trained for 50 epochs with the public code of~\cite{xiao2020generative}. In this context, sampling time refers to the duration required for the model to generate an image from the latent space. We evaluated two scenarios: generating a single image ($n=1$) and generating a batch of 128 images ($n=128$). Times are the median of 20 calls after three warm-up calls, with the GPU synchronized. The \texttt{expm\_flow\_opt15} variant is $1.22$ and $1.24$ times faster, respectively.

\begin{table}[H]
\begin{center}
 \begin{tabular}{|c || c c   |}
    \hline
      Sample & 1 sample & 128 samples \\
      \hline
      \hline
     \texttt{expm\_flow} time         &  0.0373 & 0.0555  \\
     \texttt{expm\_flow\_opt15} time &  0.0305 & 0.0447  \\
     \hline
      \hline
    Speed-up & 1.22 & 1.24 \\
    \hline
 \end{tabular}
 \end{center}
\caption{Average inference times (in seconds) for sampling a matrix
exponential variant of Glow on an NVIDIA L40S GPU. We compare the original \texttt{expm\_flow} against \texttt{expm\_flow\_opt15} for single and batched samples, including the calculated speedup factor.}\label{sampling}
 \end{table}

\subsubsection{Quality of the trained models}\label{sec:modelquality}

The proposed function changes only how the matrix exponential is evaluated,
and both \texttt{expm\_flow} and \texttt{expm\_flow\_opt15} target the same
tolerance $\varepsilon=10^{-8}$, so the trained models should be of the same
quality. To verify this, we trained the flow model with each function, using
identical configurations and random seeds, and evaluated two standard
metrics: the test negative log-likelihood in bits per dimension (BPD) and the
Fr\'echet Inception Distance (FID)~\cite{heusel2017gans} between $10{,}000$
generated images and $10{,}000$ images from the corresponding real dataset.
Lower values are better for both metrics. Following common practice for
flow models, we report the Polyak-averaged (exponential moving average)
weights~\cite{polyak1992acceleration}, which are also the weights used for
sampling, together with the raw weights for completeness. These experiments
were run on NVIDIA L40S GPUs, with the public code of~\cite{xiao2020generative} and the final implementation of \texttt{expm\_flow\_opt15}. The ImageNet datasets are the downsampled
versions of~\cite{chrabaszcz2017downsampled}. CIFAR-10 and ImageNet64 were
trained for 50 epochs and ImageNet32 for 10 epochs.

\begin{table}[H]
\begin{center}
{\footnotesize
\begin{tabular}{l|cc|cc}
\hline
 & \multicolumn{2}{c|}{Averaged weights} & \multicolumn{2}{c}{Raw weights}\\
Dataset & \texttt{expm\_flow} & \texttt{expm\_flow\_opt15} & \texttt{expm\_flow} & \texttt{expm\_flow\_opt15}\\
\hline
CIFAR-10   & 3.3239 & 3.3244 & 3.3754 & 3.3670\\
ImageNet32 & 3.7851 & 3.7777 & 3.8506 & 3.8568\\
ImageNet64 & 3.4350 & 3.4346 & 3.4478 & 3.4473\\
\hline
\end{tabular}}
\end{center}
\caption{Test bits per dimension (lower is better) of the flow model trained
with \texttt{expm\_flow} and with \texttt{expm\_flow\_opt15}, at the final
epoch.}\label{tab:bpd}
\end{table}

\begin{table}[H]
\begin{center}
{\footnotesize
\begin{tabular}{l|cc|cc}
\hline
 & \multicolumn{2}{c|}{Averaged weights} & \multicolumn{2}{c}{Raw weights}\\
Dataset & \texttt{expm\_flow} & \texttt{expm\_flow\_opt15} & \texttt{expm\_flow} & \texttt{expm\_flow\_opt15}\\
\hline
CIFAR-10   & 57.87  & 57.82  & 58.68  & 61.70\\
ImageNet32 & 59.68  & 57.86  & 65.08  & 60.08\\
ImageNet64 & 102.18 & 100.15 & 103.33 & 101.74\\
\hline
\end{tabular}}
\end{center}
\caption{FID (lower is better) between $10{,}000$ samples of the flow model
trained with each function and $10{,}000$ images of the corresponding
dataset, computed from the checkpoint with the lowest test BPD of the
averaged weights.}\label{tab:fid}
\end{table}

Tables~\ref{tab:bpd} and~\ref{tab:fid} show that the quality of the trained
models is preserved. With the averaged weights, the BPD of the two models
differs by at most $0.0074$ (ImageNet32, $0.20\%$) and by $0.0004$ on
ImageNet64. The FID values are comparable for all three
datasets, and neither function is consistently better in either metric.
Since the two functions round differently, the two training runs follow
different trajectories, and with a single run per configuration we do not
interpret these small differences as an advantage of either function. In
summary, the speedups reported in this section are obtained without degrading
the quality of the generative model.

\subsubsection{Sharpness of the bound on real flow matrices}\label{sec:boundsharpness}

To study the sharpness of the bound of Theorem~\ref{theorem2} in practice, we
compared the value $\alpha$ used by Algorithm~\ref{Algorithm4} for the
approximation of order $m=15+$,
\begin{equation}\label{eq:alpha15}
\alpha=\left(\|W^2\|_1^{8}\,\|W\|_1\right)^{1/17},
\end{equation}
with the spectral radius $\rho(W)$ and with the norm $\|W\|_1$ used by the
baseline Algorithm~\ref{Alg_expAIorig}. The test set consists of $2{,}044$
weight matrices $W$ of sizes $12$, $24$ and $48$: the weights of the
invertible $1\times1$ convolutions, which are exponentiated as $e^{W}$, of the
flow models trained in Section~\ref{sec:modelquality} on CIFAR-10, ImageNet32
and ImageNet64 with both \texttt{expm\_flow} and \texttt{expm\_flow\_opt15},
at all saved training epochs.
Since $\rho(W)\leq\alpha\leq\|W\|_1$, the ratio $\alpha/\rho(W)$ measures how
far the bound is from the spectral radius, and $\|W\|_1/\alpha$ measures the
gain with respect to the baseline.

\begin{table}[H]
\begin{center}
{\footnotesize
\begin{tabular}{l|cccc}
\hline
Ratio & Mean & Median & Minimum & Maximum\\
\hline
$\alpha/\rho(W)$     & 1.81 & 1.79 & 1.28 & 2.41\\
$\|W\|_1/\rho(W)$    & 3.49 & 3.22 & 1.98 & 6.04\\
$\|W\|_1/\alpha$     & 1.91 & 1.83 & 1.36 & 2.71\\
\hline
\end{tabular}}
\end{center}
\caption{Value $\alpha$ from \eqref{eq:alpha15} compared with the spectral
radius $\rho(W)$ and with $\|W\|_1$, for $2{,}044$ weight matrices of the
trained flow models.}\label{tab:boundsharpness}
\end{table}

Table~\ref{tab:boundsharpness} shows that $\alpha$ is on average $1.81$ times
the spectral radius and never more than $2.41$ times it, whereas $\|W\|_1$ is
on average $3.49$ times the spectral radius and up to $6.04$ times it. On
average, $\alpha$ is therefore $1.91$ times smaller than $\|W\|_1$. Since the
scaling parameter grows with $\log_2$ of the bound, and the mean of
$\log_2(\|W\|_1/\alpha)$ over the test set is $0.91$, this saves on average
about one squaring with respect to scaling by $\|W\|_1$, for the matrices
that require scaling. The remaining gap to $\rho(W)$ reflects
that these weight matrices are nonnormal, for which $\|W^k\|^{1/k}$ approaches
$\rho(W)$ only for large~$k$.

\section{Conclusions}\label{sec:conclusions}

This work presented a Taylor-based algorithm for computing the matrix exponential, utilizing advanced polynomial evaluation schemes that extend beyond the classical Paterson--Stockmeyer technique. The proposed method incorporates a self-contained framework for the dynamic selection of the Taylor order and scaling factor, based on a rigorous error analysis designed to maintain accuracy for any user-specified tolerance $\varepsilon$ greater than or equal to the working unit roundoff. This approach distinguishes itself from traditional implementations that rely on fixed-precision tables and from complex variable-precision algorithms requiring specialized external libraries. When employing IEEE double precision arithmetic ($u=2^{-53}\approx 1.11 \times 10^{-16} \ll 10^{-8}$ from \cite[Sec. 6.4]{xiao2020generative}), achieving accuracy near machine precision requires adjusting the algorithm parameters by selecting more terms in the Taylor approximation \cite[Sec. 6]{SID19}. Implementing this specific adaptation remains a direction for future work.

The algorithm was specifically evaluated within the context of generative AI flows, where the matrix exponential constitutes a key computational requirement for modeling invertible transformations. By leveraging refined error bounds for nonnormal matrices and streamlined operational logic, the proposed method provides a more efficient and accurate alternative to previous dynamic schemes, as well as to standard library implementations such as \texttt{torch.linalg.matrix\_exp} \cite{BBS19} and \texttt{scipy.linalg.expm} \cite{AlHi09}. Numerical evaluations on a diverse testbed of matrices and practical generative models demonstrate that, relative to the baseline, the proposed Taylor-based approach requires approximately half the matrix products, computes the matrix exponential up to $1.9$ times faster during training, and preserves the quality of the trained models. Because matrix exponentials account for approximately one-fifth of the training time in the considered models, the overall training time reduction is proportionally smaller, while sampling accelerates by a factor of $1.2$.

These results demonstrate that modern polynomial evaluation strategies enable higher-order approximations at a lower computational cost, measured in terms of matrix multiplications, than traditional Paterson--Stockmeyer-based methods. Consequently, the proposed implementation offers a portable, library-independent solution for high-throughput applications.

\section*{Acknowledgements}
The authors would like to thank Roger B. Dannenberg for providing valuable feedback on the writing of this manuscript.

\vspace{0.5em} 
\begin{center} \textit{In loving memory of Aurora, a true source of light and joy to this world.} 
\end{center}

\section*{Author contributions}
J. Sastre: Conceptualization, Methodology, Software, Formal analysis, Funding acquisition, Project administration, Writing - original draft, Writing - review \& editing.

D. Faronbi: Software, Validation, Investigation, Data curation,  Resources, Visualization, Writing - original draft, Writing - review \& editing.

J.M. Alonso: Methodology, Software, Formal analysis, Validation, Investigation, Data curation, Resources, Visualization, Writing - original draft, Writing - review \& editing.

P. Traver: Software, Investigation, Validation, Resources, Visualization, Writing - review \& editing.

J. Ib\'a\~nez: Methodology, Formal analysis, Investigation, Data curation, Visualization, Supervision, Writing - review \& editing.

N. Lloret.: Resources, Validation, Supervision, Funding acquisition, Project administration, Writing - review \& editing.

\section*{Declaration of competing interest} The authors declare that they have no known competing financial interests or personal relationships that could have appeared to influence the work reported in this paper.

\section*{Declaration of generative AI and AI-assisted technologies in the writing process}
During the preparation of this work the authors used Gemini in order
to improve the language and readability of the manuscript. After
using this tool, the authors reviewed and edited the content as
needed and take full responsibility for the content of the published
article.


\section{References} \label{references}

\begin{thebibliography}{00}

\bibitem{High08} N.J. Higham, Functions of Matrices: Theory and Computation, SIAM, Philadelphia, PA, 2008.

\bibitem{MoVa03} C.B. Moler, C.F. Van Loan, Nineteen dubious ways to compute the exponential of a matrix, twenty-five years later, SIAM Rev. 45 (2003) 3--49.

\bibitem{SIDR15} J. Sastre, J.J. Ib\'a\~nez, E. Defez, P.A. Ruiz, Efficient scaling-squaring Taylor method for computing matrix exponential, SIAM J. Sci. Comput. 37 (1) (2015) A439--A455.

\bibitem{SIDR11b} J. Sastre, J. Ib\'a\~nez, E. Defez, P. Ruiz, Accurate matrix exponential computation to solve coupled differential models in engineering, Math. Comput. Model. 54 (2011) 1835--1840.

\bibitem{SIDR14} J. Sastre, J. Ib\'a\~nez, E. Defez, P. Ruiz, Accurate and efficient matrix exponential computation, Int. J. Comput. Math. 91 (1) (2014) 97--112.

\bibitem{RSID16} P. Ruiz, J. Sastre, J. Ib\'a\~nez, E. Defez, High performance computing of the matrix exponential, J. Comput. Appl. Math. 291 (2016) 370--379.

\bibitem{Sastre18} J. Sastre, Efficient evaluation of matrix polynomials, Linear Algebra Appl. 539 (2018) 229--250.

\bibitem{SaIb21} J. Sastre, J. Ib\'a\~nez, Evaluation of matrix polynomials beyond the Paterson--Stockmeyer method, Mathematics 9 (2021) 1600.

\bibitem{Sastre12} J. Sastre, Efficient mixed rational and polynomial approximation of matrix functions, Appl. Math. Comput. 218 (24) (2012) 11938--11946.

\bibitem{BlDo99} S. Blackford, J. Dongarra, Installation guide for LAPACK, LAPACK Working Note 41, Univ. of Tennessee, 1999.

\bibitem{PaSt73} M.S. Paterson, L.J. Stockmeyer, On the number of nonscalar multiplications necessary to evaluate polynomials, SIAM J. Comput. 2 (1) (1973) 60--66.

\bibitem{BBS19} P. Bader, S. Blanes, F. Casas, Computing the matrix exponential with an optimized Taylor polynomial approximation, Mathematics 7 (12) (2019) 1174.

\bibitem{BKS25} S. Blanes, N. Kopylov, M. Seydao\u{g}lu, Efficient Scaling and Squaring Method for the Matrix Exponential, SIAM J. Matrix Anal. Appl. 46 (1) (2025) 74--93.

\bibitem{FaHi19} M. Fasi, N.J. Higham, An Arbitrary Precision Scaling and Squaring Algorithm for the Matrix Exponential, SIAM J. Matrix Anal. Appl. 40 (4) (2019) 1233--1256.

\bibitem{CaZi18} M. Caliari, F. Zivcovich, On-the-fly backward error estimate for matrix exponential approximation by Taylor algorithm, Appl. Math. Comput. 338 (2018) 526--535.

\bibitem{Jar2021} E. Jarlebring, M. Fasi, E. Ringh, Computational graphs for matrix functions, ACM Trans. Math. Softw. 48 (4) (2023) 1--35.

\bibitem{JSI2025} E. Jarlebring, J. Sastre, J. Ib\'a\~nez, Polynomial approximations for the matrix logarithm with computation graphs, Linear Algebra Appl. 721 (2025) 692--714.

\bibitem{AlHi09} A.H. Al-Mohy, N.J. Higham, A new scaling and squaring algorithm for the matrix exponential, SIAM J. Matrix Anal. Appl. 31 (3) (2009) 970--989.

\bibitem{SID19} J. Sastre, J. Ib\'a\~nez, E. Defez, Boosting the computation of the matrix exponential, Appl. Math. Comput. 340 (2019) 206--220.

\bibitem{High05} N.J. Higham, The scaling and squaring method for the matrix exponential revisited, SIAM J. Matrix Anal. Appl. 26 (4) (2005) 1179--1193.

\bibitem{DoMo02} E.D. Dolan, J.J. Mor\'e, Benchmarking optimization software with performance profiles, Math. Program. 91 (2002) 201--213.

\bibitem{xiao2020generative} C. Xiao, L. Liu, Generative flows with matrix exponential, in: Proc. 37th Int. Conf. Mach. Learn., PMLR, 2020, pp. 10452--10461.

\bibitem{QuickExp20} J. Corwell, W.D. Blair, Industry Tip: Quick and Easy Matrix Exponentials, IEEE Aerosp. Electron. Syst. Mag. 35 (5) (2020) 49--52.

\bibitem{copet2023simple} J. Copet, F. Kreuk, I. Gat, T. Remez, D. Kant, G. Synnaeve, Y. Adi, A. D\'efossez, Simple and controllable music generation, in: Adv. Neural Inf. Process. Syst., vol. 36, 2023, pp. 47704--47720.

\bibitem{wu2024music} S.L. Wu, C. Donahue, S. Watanabe, N.J. Bryan, Music ControlNet: Multiple Time-Varying Controls for Music Generation, IEEE/ACM Trans. Audio Speech Lang. Process. 32 (2024) 2692--2703.

\bibitem{bengio2013generalized} Y. Bengio, L. Yao, G. Alain, P. Vincent, Generalized denoising auto-encoders as generative models, in: Adv. Neural Inf. Process. Syst., vol. 26, 2013.

\bibitem{jahanian2021generative} A. Jahanian, X. Puig, Y. Tian, P. Isola, Generative models as a data source for multiview representation learning, 2021, arXiv:2106.05258.

\bibitem{goodfellow2014generative} I.J. Goodfellow, J. Pouget-Abadie, M. Mirza, B. Xu, D. Warde-Farley, S. Ozair, A. Courville, Y. Bengio, Generative adversarial nets, in: Adv. Neural Inf. Process. Syst., vol. 27, 2014.

\bibitem{kingma2013auto} D.P. Kingma, M. Welling, Auto-encoding variational Bayes, 2013, arXiv:1312.6114.

\bibitem{ho2020denoising} J. Ho, A. Jain, P. Abbeel, Denoising diffusion probabilistic models, in: Adv. Neural Inf. Process. Syst., vol. 33, 2020, pp. 6840--6851.

\bibitem{li66} M. Liou, A novel method of evaluating transient response, Proc. IEEE 54 (1966) 20--23.

\bibitem{higham2002test} N.J. Higham, The Matrix Computation Toolbox, 2025. https://es.mathworks.com/matlabcentral/fileexchange/2360-the-matrix-computation-toolbox (accessed 22 December 2025).

\bibitem{wright2009eigtool} T.G. Wright, Eigtool, version 2.1, 2009. http://www.comlab.ox.ac.uk/pseudospectra/eigtool (accessed 22 December 2025).

\bibitem{krizhevsky2009learning} A. Krizhevsky, G. Hinton, Learning multiple layers of features from tiny images, Technical Report, University of Toronto, 2009.

\bibitem{pmlr-v48-oord16} A. van den Oord, N. Kalchbrenner, K. Kavukcuoglu, Pixel Recurrent Neural Networks, in: Proc. 33rd Int. Conf. Mach. Learn., PMLR, 2016, pp. 1747--1756.

\bibitem{kingma2018glow} D.P. Kingma, P. Dhariwal, Glow: Generative flow with invertible 1x1 convolutions, in: Adv. Neural Inf. Process. Syst., vol. 31, 2018.

\bibitem{papamakarios2017masked} G. Papamakarios, T. Pavlakou, I. Murray, Masked autoregressive flow for density estimation, in: Adv. Neural Inf. Process. Syst., vol. 30, 2017.

\bibitem{dinh2016density} L. Dinh, J. Sohl-Dickstein, S. Bengio, Density estimation using Real NVP, 2016, arXiv:1605.08803.

\bibitem{chen2019residual} R.T.Q. Chen, J. Behrmann, D.K. Duvenaud, J.H. Jacobsen, Residual flows for invertible generative modeling, in: Adv. Neural Inf. Process. Syst., vol. 32, 2019.

\bibitem{hoogeboom2020convolution}
E. Hoogeboom, V. Garcia Satorras, J.M. Tomczak, M. Welling, The convolution exponential and generalized Sylvester flows, in: Advances in Neural Information Processing Systems (NeurIPS), 2020, arXiv preprint arXiv:2006.01910.

\bibitem{lezcano2019cheap}
M. Lezcano-Casado, D. Martinez-Rubio, Cheap orthogonal constraints in neural networks, in: International Conference on Machine Learning (ICML), 2019, arXiv preprint arXiv:1901.08428.

\bibitem{bilos2021neural}
M. Bilos, et al., Neural flows: Efficient alternative to Neural ODEs, in: Advances in Neural Information Processing Systems (NeurIPS), 2021, arXiv preprint arXiv:2110.13040.

\bibitem{dockhorn2022score}
T. Dockhorn, A. Vahdat, K. Kreis, Score-based generative modeling with critically-damped Langevin diffusion, in: International Conference on Learning Representations (ICLR), 2022, arXiv preprint arXiv:2112.07068.

\bibitem{zhang2023fast}
Q. Zhang, Y. Chen, Fast sampling of diffusion models with exponential integrator, in: The Eleventh International Conference on Learning Representations (ICLR), 2023, arXiv preprint arXiv:2204.13902.



\bibitem{zhang2023gddim}
Q. Zhang, M. Tao, Y. Chen, gDDIM: Generalized denoising diffusion implicit models, in: The Eleventh International Conference on Learning Representations (ICLR), 2023, arXiv preprint arXiv:2206.05564.

\bibitem{gu2023mamba} A. Gu, T. Dao, Mamba: Linear-time sequence modeling with selective state spaces, in: First Conference on Language Modeling (COLM), 2024, arXiv preprint arXiv:2312.00752.

\bibitem{smith2023simplified}
J.T.H. Smith, A. Warrington, S. Linderman, Simplified state space layers for sequence modeling, in: The Eleventh International Conference on Learning Representations (ICLR), 2023, arXiv preprint arXiv:2208.04933.

\bibitem{yu2024robustifying}
A. Yu, A. Nigmetov, D. Morozov, M.W. Mahoney, N.B. Erichson, Robustifying state-space models for long sequences via approximate diagonalization, in: The Twelfth International Conference on Learning Representations (ICLR), 2024, arXiv preprint arXiv:2310.01698.

\bibitem{sherry2025flow}
F.M. Sherry, B.M.N. Smets, Flow matching on Lie groups, in: F. Nielsen, F. Barbaresco (Eds.), Geometric Science of Information (GSI 2025), Springer Nature Switzerland, Cham, 2026, pp. 54--62, arXiv preprint arXiv:2504.00494 (2025)

\bibitem{davis2026generalised}
O. Davis, M.S. Albergo, N.M. Boffi, M.M. Bronstein, A.J. Bose, Generalised flow maps for few-step generative modelling on Riemannian manifolds, in: The Fourteenth International Conference on Learning Representations (ICLR), 2026, arXiv preprint arXiv:2510.21608.

\bibitem{kingma2014adam} D.P. Kingma, J. Ba, Adam: A method for stochastic optimization, in: Int. Conf. Learn. Represent. (ICLR), 2015, arXiv:1412.6980.

\bibitem{heusel2017gans} M. Heusel, H. Ramsauer, T. Unterthiner, B. Nessler, S. Hochreiter, GANs trained by a two time-scale update rule converge to a local Nash equilibrium, in: Adv. Neural Inf. Process. Syst., vol. 30, 2017.

\bibitem{polyak1992acceleration} B.T. Polyak, A.B. Juditsky, Acceleration of stochastic approximation by averaging, SIAM J. Control Optim. 30 (4) (1992) 838--855.

\bibitem{chrabaszcz2017downsampled} P. Chrabaszcz, I. Loshchilov, F. Hutter, A downsampled variant of ImageNet as an alternative to the CIFAR datasets, 2017, arXiv:1707.08819.

\bibitem{advanpix}
Advanpix, Multiprecision computing toolbox, 2022. http://www.advanpix.com (accessed 18 October 2026).
\end{thebibliography}
\end{document}